\documentclass[twoside,11pt]{article}

\usepackage{blindtext}

\usepackage{jmlr2e}

\usepackage{color}
\usepackage[normalem]{ulem}
\usepackage{threeparttable}

\usepackage{siunitx}
\usepackage{amsmath}
\usepackage{algorithm}
\usepackage{algpseudocode}
\usepackage{bm}
\usepackage{tikz}
\usetikzlibrary{hobby}
\usetikzlibrary{decorations.text}
\usetikzlibrary{bayesnet}
\usepackage{wrapfig}
\usepackage{booktabs}       % professional-quality tables
\usepackage{multirow}
\usepackage{amsfonts}       % blackboard math symbols
\usepackage{nicefrac}       % compact symbols for 1/2, etc.
\usepackage{microtype}      % microtypography
\newcommand{\stkout}[1]{\ifmmode\text{\sout{\ensuremath{#1}}}\else\sout{#1}\fi}

\DeclareMathOperator*{\argmax}{arg\,max}
\DeclareMathOperator*{\argmin}{arg\,min}

\usepackage{lastpage}
\jmlrheading{25}{2024}{1-\pageref{LastPage}}{4/22; Revised 4/23}{5/24}{22-0348}{Cheng Zhang and Frederick A. Matsen IV}

\ShortHeadings{Variational Bayes Phylogenetic Inference}{Zhang and Matsen IV}
\firstpageno{1}

\begin{document}

\title{A Variational Approach to Bayesian Phylogenetic Inference}

\author{\name Cheng Zhang\thanks{Corresponding author.} \email chengzhang@math.pku.edu.cn \\
       \addr School of Mathematical Sciences and Center for Statistical Science\\
       Peking University\\
       Beijing, 100871, China
       \AND
       \name Frederick A.\ Matsen IV \email matsen@fredhutch.org \\
       \addr Howard Hughes Medical Institute\\
       Computational Biology Program\\
       Fred Hutchinson Cancer Research Center\\
       Seattle, WA 98109, USA\\
       Department of Genome Sciences and Department of Statistics\\
       University of Washington\\
       Seattle, WA 98195, USA
       }

\editor{Chris Wiggins}

\maketitle

\begin{abstract}%   <- trailing '%' for backward compatibility of .sty file
Bayesian phylogenetic inference is currently done via Markov chain Monte Carlo (MCMC) with simple proposal mechanisms.
This hinders exploration efficiency and often requires long runs to deliver accurate posterior estimates.
In this paper, we present an alternative approach: a variational framework for Bayesian phylogenetic analysis.
We propose combining subsplit Bayesian networks, an expressive graphical model for tree topology distributions, and a structured amortization of the branch lengths over tree topologies for a suitable variational family of distributions.
We train the variational approximation via stochastic gradient ascent and adopt gradient estimators for continuous and discrete variational parameters separately to deal with the composite latent space of phylogenetic models.
We show that our variational approach provides competitive performance to MCMC, while requiring much fewer (though more costly) iterations due to a more efficient exploration mechanism enabled by variational inference.
Experiments on a benchmark of challenging real data Bayesian phylogenetic inference problems demonstrate the effectiveness and efficiency of our methods.
\end{abstract}


\section{Introduction}
Bayesian phylogenetic inference is an essential statistical method for modern molecular evolutionary analysis and has been used in a wide range of applications such as genomic epidemiology \citep{Dudas2017-sb, Du_Plessis2021-tq} and conservation genetics \citep{DeSalle2004-mm}.
Given aligned sequence data (e.g., DNA, RNA or protein sequences) and a model of evolution, Bayesian phylogenetics provides principled approaches to quantify the uncertainty of the evolutionary process in terms of the posterior probabilities of phylogenetic trees \citep{Huelsenbeck01b}.
In addition to uncertainty quantification, Bayesian methods enable integrating out tree uncertainty in order to get more confident estimates of parameters of interest, such as factors governing the transmission of Ebolavirus \citep{Dudas2017-sb}.
Bayesian methods also allow complex substitution models \citep{Lartillot2004-qf}, which are important in elucidating deep phylogenetic relationships \citep{Feuda2017-kx}.

Ever since its introduction to the phylogenetic community in the 1990s, Bayesian phylogenetic inference has been dominated by random-walk Markov chain Monte Carlo (MCMC) approaches \citep{Yang97, Mau99, Huelsenbeck01, Drummond02, Drummond05}.
However, this approach is fundamentally limited by the complexities of tree space.
A typical MCMC method for phylogenetic inference involves two steps in each iteration:
first, a new tree is proposed by randomly perturbing the current tree,
and second, the tree is accepted or rejected according to the Metropolis-Hastings acceptance probability.
Any such random walk algorithm faces obstacles in the phylogenetic case, in which the high-posterior trees are a tiny fraction of the combinatorially exploding number of trees.
Thus, major modifications of trees are likely to be rejected, restricting MCMC tree movement to local modifications that may have difficulty moving between multiple peaks in the posterior distribution \citep{Whidden2015-eq}.
Although recent MCMC methods for distributions on Euclidean space use intelligent proposal mechanisms such as Hamiltonian Monte Carlo \citep{Neal2011-yo}, it is not straightforward to extend such algorithms to the composite structure of tree space, which includes both tree topology (discrete object) and branch lengths (continuous positive vector) \citep{Dinh2017-oj}.

Variational inference (VI) is an alternative approximate inference method for Bayesian analysis which is gaining in popularity \citep{Jordan99, Wainwright08, Blei17}.
Unlike MCMC methods that sample from the posterior, VI seeks the best candidate from a family of tractable distributions that minimizes a statistical distance measure to the target posterior, usually the Kullback-Leibler (KL) divergence.
By reformulating the inference problem into an optimization problem, VI tends to be faster and easier to scale to large data (via stochastic gradient descent) \citep{Blei17}.
However, VI can also introduce a large bias if the variational family of distributions is insufficiently flexible.
The success of variational methods, therefore, largely relies on the construction of appropriate tractable variational family of distributions and efficient optimization procedures.

Recent years have witnessed several efforts on VI for Bayesian phylogenetics that focus on continuous parameters (e.g., the branch lengths), where either a fixed tree topology is assumed \citep{Fourment19,Ki2022-nt} or MCMC algorithms (e.g., Gibbs sampling) are employed to handle the uncertainty of tree topologies \citep{Dang19}.
To our knowledge, there have been no previous full variational formulations of Bayesian phylogenetic inference that deal with the joint parameter space.
This has been due to the lack of an appropriate family of distributions on phylogenetic tree topologies that can be used as variational distributions.
However the prospects for variational inference have changed recently with the introduction of \emph{subsplit Bayesian networks} (SBNs) \citep{SBN}, which provide a family of flexible distributions on tree topologies (i.e.\ trees without branch lengths).
SBNs build on previous work \citep{Hhna2012-pm,Larget2013-et}, but in contrast to these previous efforts, SBNs are sufficiently flexible for real Bayesian phylogenetic posteriors \citep{SBN}.
Other families of distributions over the joint phylogenetic parameter space have also been proposed outside of the variational framework~\citep{Holder14,Baele16} with the goal of more efficient marginal likelihood estimation.

In this paper, we develop a general variational inference framework for Bayesian phylogenetics.
We show that SBNs, when combined with appropriate approximations for branch length distributions, can provide flexible variational approximations over the joint latent space of phylogenetic trees with branch lengths.
This, combined with other variational distributions when needed (e.g., distributions of population size parameters), provides a suitable variational family of distributions for Bayesian phylogenetic inference.
The variational approximation can be trained using stochastic gradient ascent via efficient unbiased gradient estimators proposed recently for discrete and continuous parameters \citep{VAE, RWS, VIMCO}.
We introduce several effective approaches to provide practical subsplit support estimation for SBNs that significantly reduce the number of parameters while maintaining adequate approximation power for tree topology posterior estimation.
We also leverage the similarity of local topological structures for a shared parameterization of the branch length distributions over different tree topologies to further reduce the complexity in variational parameterization and ease optimization.
With guided exploration in the tree topology space within that support (enabled by SBNs) and joint learning of the branch length distributions across tree topologies (via structured amortization), our variational approach provides competitive performance to MCMC with much fewer (though more costly) iterations.
Moreover, the variational approximations can be readily used for downstream statistical analysis such as marginal likelihood estimation for model comparison via importance sampling.
We demonstrate the effectiveness and efficiency of our methods on a benchmark of challenging real data Bayesian phylogenetic inference problems.
A preliminary version of this work has appeared in \citet{VBPI}; in addition to providing a complete exposition, this version provides additional benchmarks and extends our work to an additional class of models (``time trees'') which are very popular for analyzing viral sequences.

\section{Background}
In this section, we provide a brief overview of Bayesian phylogenetics and variational inference, which are the fundamental ingredients for this paper.
\subsection{Bayesian Phylogenetics}
The core in Bayesian phylogenetics is a probability model that describes a generative process for a given data set consisting of nucleotide or amino acid sequences from a collection of \emph{taxa} (singular \emph{taxon}, e.g., species).
There are two essential ingredients for this probability model and the associated parameters.
The first ingredient is a phylogenetic tree which contains a tree topology $\tau$ and the associated non-negative branch lengths $\bm{q}$.
The tree topology $\tau$ is a bifurcating tree whose leaves are labelled by the taxa.
It is used to model the evolutionary relationship of the observed species, where the internal nodes represents the unobserved characters (e.g., DNA bases) of the ancestral species.
Phylogenetic trees can be \emph{rooted}, in which a root position in the tree is chosen to be the most ancestral part of the tree, or \emph{unrooted}, if only evolutionary relationships and not the direction of evolution is of interest.
The second ingredient is a continuous-time Markov model that is used to describe the transition probabilities of the characters along the branches of the tree.
These transition probabilities are assumed to be independent across the tree.
In this paper, we focus on the parameter space of phylogenetic trees and assume the Markov evolution model is known (e.g., the \citet{Jukes69} model).
The phylogenetic likelihood is, therefore, the function of phylogenetic tree $\tau,\bm{q}$ that gives the probability of the observed data.
Let $\bm{Y}=\{Y_1,Y_2,\ldots,Y_M\}\in \Omega^{N\times M}$ be the observed sequences (with characters in $\Omega$) of length $M$ over $N$ species.
Let $Y_i(v)$ is the observed character of node $v$ at site $i$.
Define an extension of $Y_i$ to the internal nodes to be an assignment of states $a^i$ to the nodes of the tree such that $a^i_v = Y_i(v)$ for every tip node $v$.
The probability of each site observation $Y_i$ is defined as the marginal distribution over the leaves
\begin{equation}\label{eq:marginal_prob}
p(Y_i|\tau,\bm{q}) = \sum_{a^i}\eta(a^i_{r})\prod_{(u,v)\in E(\tau)}P_{a^i_ua^i_v}(q_{uv})
\end{equation}
where $r$ is the root node, $a^i$ ranges over all extensions of $Y_i$ to the internal nodes with $a^i_u$ being the assigned character of node $u$, $E(\tau)$ denotes the set of edges of $\tau$, $q_{uv}$ is the branch length for the edge $(u,v)\in E(\tau)$, and $P_{ij}(t)$ denotes the transition probability from character $i$ to character $j$ across an edge of branch length $t$ given by the Markov model whose stationary distribution is $\eta$.
For unrooted trees with a time reversible Markov model, \eqref{eq:marginal_prob} also defines a valid marginal probability where $r$ can be any internal node due to the so-called ``pulley principle'' \citep{Felsenstein81}.

Assuming different sites are identically distributed and evolve independently, the likelihood function then takes the following form
\begin{equation}\label{eq:phylo_likelihood}
p(\bm{Y}|\tau, \bm{q}) = \prod_{i=1}^Mp(Y_i|\tau,\bm{q}) = \prod_{i=1}^M \sum_{a^i}\eta(a^i_{r})\prod_{(u,v)\in E(\tau)}P_{a^i_ua^i_v}(q_{uv})
\end{equation}
The phylogenetic likelihood \eqref{eq:phylo_likelihood} can be evaluated efficiently through the pruning algorithm \citep{Felsenstein81}, also known as the sum-product algorithm in general probabilistic graphical models \citep{Strimmer2000, Koller09, Hohna14}.
Given a proper prior distribution with density $p(\tau, \bm{q})$ of the tree topology and the branch lengths, the posterior $p(\tau,\bm{q}|\bm{Y})$ is proportional to the joint density
\begin{equation}\label{eq:phylo_posterior}
p(\tau, \bm{q}|\bm{Y}) = \frac{p(\bm{Y}|\tau,\bm{q})p(\tau,\bm{q})}{p(\bm{Y})} \propto p(\bm{Y}|\tau,\bm{q})p(\tau,\bm{q})
\end{equation}
where $p(\bm{Y})$ is the unknown normalizing constant.
Bayesian phylogenetics then amounts to properly estimating the phylogenetic posterior in~\eqref{eq:phylo_posterior}, which is typically done via Markov chain Monte Carlo (MCMC) methods.

\subsection{Variational Inference}
%As an alternative to MCMC, variational inference (VI) is another approximate Bayesian inference method that is growing in popularity \citep{Jordan99, Wainwright08, Blei17}.
Rather than sampling from the posterior as in MCMC, VI turns the posterior inference problem into an optimization problem.
Suppose $\bm{Y}$ is the observed data and $p(\bm{Y}|\bm{\theta})$ is the probability model where $\bm{\theta}$ is the set of model parameters.
%To estimate the posterior $p(\bm{\theta}|\bm{Y})$, VI tries to find the closest member from a family of densities over $\bm{\theta}$, parameterized by learnable ``variational parameters''.
The objective of VI is to find the closest member from a family of densities over $\bm{\theta}$ to the posterior $p(\bm{\theta}|\bm{Y})$.
Let $q_{\bm{\phi}}(\theta)$ denotes the density of variational distribution where $\bm{\phi}$ is the set of variational parameters.
To measure the ``closeness'' between $p(\bm{\theta}|\bm{Y})$ and $q_{\bm{\phi}}(\bm{\theta})$, a divergence measure $D$ between distributions is often employed such that $D(q_{\bm{\phi}}(\bm{\theta})\|p(\bm{\theta}|\bm{Y}) \geq 0$ and $D(q_{\bm{\phi}}(\bm{\theta})\|p(\bm{\theta}|\bm{Y})) = 0$ iff $q_{\bm{\phi}}(\bm{\theta})=p(\bm{\theta}|\bm{Y})$.
While various divergence measures exist \citep{Amari85, Amari09, Minka05, Sason16}, the most commonly used divergence in VI is the Kullback-Leibler (KL) divergence \citep{Kullback51}
\[
D_{\mathrm{KL}}(q_{\bm{\phi}}(\bm{\theta})\|p(\bm{\theta}|\bm{Y})) = \int q_{\bm{\phi}}(\bm{\theta}) \log\left(\frac{q_{\bm{\phi}}(\bm{\theta})}{p(\bm{\theta}|\bm{Y})}\right)d\bm{\theta}.
\]
Inference now amounts to solving the following optimization problem
\begin{equation}\label{eq:vi}
\bm{\phi}^\ast = \argmin_{\bm{\phi}}D_{\mathrm{KL}}(q_{\bm{\phi}}(\bm{\theta})\|p(\bm{\theta}|\bm{Y})).
\end{equation}
The best approximation $q_{\bm{\phi}^\ast}(\bm{\theta})$ then serves as a proxy for the exact posterior.

As the exact posterior $p(\bm{\theta}|\bm{Y})$ is only known up to a constant $p(\bm{Y})$, the objective in \eqref{eq:vi} is not computable.
In practice, we optimize an alternative function
\[
L(\bm{\phi}) = \int q_{\bm{\phi}}(\bm{\theta}) \log\left(\frac{p(\bm{\theta},\bm{Y})}{q_{\bm{\phi}}(\bm{\theta})}\right)d\bm{\theta} = \log p(\bm{Y}) - D_{\mathrm{KL}}(q_{\bm{\phi}}(\bm{\theta})\|p(\bm{\theta}|\bm{Y})) \leq \log p(\bm{Y}).
\]
This function is called the evidence lower bound (ELBO). As the KL divergence is the gap between the constant $\log p(\bm{Y})$ and the ELBO, maximizing the EBLO is equivalent to minimizing the KL divergence.

To complete the specification of the optimization problem in \eqref{eq:vi}, we need to prescribe a variational family of distributions.
A common choice is the mean-field variational family, where different components of $\bm{\theta}$ are mutually independent.
Given conditionally conjugate exponential family models for $p(\bm{\theta},\bm{Y})$, mean-field variational inference can be solved efficiently via coordinate ascent \citep{PRML}.
However, the independence assumption also leads to insufficient approximation accuracy especially when the exact posterior variables are highly correlated.
Many more flexible variational families have been proposed \citep{Saul96, Bishop98, Barber99, NF, IAF, RealNVP}, albeit the resulting optimization problems are often more difficult to solve.
Furthermore, researchers have also generalized VI beyond conditional conjugate models, and designed generic model-agnostic methods based on stochastic optimization via Monte Carlo gradient estimates \citep{Nott12, Paisley12, VAE, Rezende14, Ranganath13}.
Overall, designing an appropriate variational family that strikes a good balance between expressiveness and complexity in optimization has been one of the key obstacles to the success of VI.
In what follows, we first describe an essential ingredient for the construction of a suitable variational family for Bayesian phylogenetics.

\section{Subsplit Bayesian Networks}\label{sec:sbn}

Since the number of tree topologies explodes combinatorially as the number of taxa increases, designing tractable and flexible distributions over the discrete tree topology space turns out to be quite challenging.
Subsplit Bayesian networks, as recently proposed by \citet{SBN}, supply a powerful probabilistic graphical model that can provide a flexible family of distributions over tree topologies.
The success of SBNs lies in its proper utilization of the similarity of the local topological structures of trees, i.e., \emph{subsplits}, as defined below:

\begin{definition}[Subsplit]
Let $\mathcal{X}$ be a set of labeled leaves that represent the observed taxa. We call a nonempty subset $C$ of $\mathcal{X}$ a \emph{clade} and denote the set of all clades of $\mathcal{X}$ as $\mathcal{C}=\{C\subset \mathcal{X}: C\neq\emptyset\}$. A nonempty subset of a clade $C$ is called a subclade of $C$.
Let $\succ$ be a total order on $\mathcal{C}$ (e.g., lexicographical order). A \emph{subsplit} $(Y,Z)$ of a clade $C$ is an ordered pair of disjoint subclades of $C$ such that $Y\cup Z=C$ and $Y\succ Z$.
\end{definition}

\begin{figure}[t!]
\begin{center}
\input{figs/bn-new.tex}
\caption{
Subsplit Bayesian networks and a simple example for a leaf set of 4 taxa (represented by $A$, $B$, $C$, and $D$ respectively).
({\bf Left}): General subsplit Bayesian networks. The solid backbone represents the complete and binary tree network $\mathcal{B}_\mathcal{X}^\ast$. The dashed arrows represent the additional dependence for more expressive SBNs.
%A leaf label set $\mathcal{X}$ of 4 species, each label corresponds to a DNA sequence.
({\bf Middle(left)}): Examples of phylogenetic trees (rooted) that are hypothesized to model the evolutionary history of the taxa.
({\bf Middle(right)}): The corresponding node assignments for the trees.
For ease of illustration, subsplit $(Y,Z)$ is represented as $\frac{Y}{Z}$ in the graph.
The \emph{dashed gray subgraphs} represent trivial splitting processes where the subsplits are deterministically assigned, and are used purely to complement the networks such that the overall network has a fixed structure.
({\bf Right}): The SBN for these examples, which is $\mathcal{B}_\mathcal{X}^\ast$ in this case.}
\label{fig:sbn}
\end{center}
\vspace{-0.15in}
\end{figure}

The hierarchical structure of a rooted tree topology induces a splitting process from the root node to the tip nodes, where each internal node has a subsplit that represents its local splitting pattern.
Following this splitting process, all the topological information of the tree can be naturally converted into a sequence of subsplits.
As the tree topology can be easily recovered from these local splitting patterns afterward, we call them the \emph{subsplit decomposition} of the tree.
For example, in Figure~\ref{fig:sbn}, the top tree has a subsplit decomposition $\{(\{A,B,C\},\{D\}), (\{A\},\{B,C\}), (\{B\},\{C\})\}$ and the bottom tree has a subsplit decomposition $\{(\{A,B\}, \{C,D\}), (\{A\},\{B\}), (\{C\},\{D\})\}$.
Unlike the full tree topologies that require a combinatorially exploding representation space, subsplit decompositions allow us to represent topological information in a more compact manner since subsplits (and other local splitting patterns) can be shared across different tree topologies, as shown in the following example.

\begin{example}\label{example:subsplit_eff}
Consider a leaf label set $\mathcal{X}$ that has $3n$ taxa $A_1,A_2,\ldots, A_{3n}$.
Let $\tau_0$ be a tree topology with $n$ tips.
Now partition $\mathcal{X}$ into $n$ clades with $3$ taxa, say $C_1=\{A_1, A_2, A_3\}$, $C_2=\{A_4, A_5, A_6\}$, $\ldots, C_n=\{A_{3n-2}, A_{3n-1}, A_{3n}\}$ and assign them to the tip nodes of $\tau_0$.
Note that each clade $C_i, 1\leq i \leq n$ may have $3$ different local topologies.
Overall, this allows us to represent $3^n$ tree topologies on $\mathcal{X}$.
On the other hand, the subsplits from these tree topologies are either the subsplits induced by $\tau_0$ (which is $\mathcal{O}(n)$) or the local subsplits from those $3$ taxa subtrees in the tips (which is also $\mathcal{O}(n)$).
Therefore, we can represent an exponential number of tree topologies with $\mathcal{O}(n)$ subsplits. See Figure \ref{fig:subsplit_eff} for an illustration.
\end{example}

\begin{figure}[t!]
\begin{center}
\input{figs/subsplit-efficiency.tex}
\caption{An illustration for Example \ref{example:subsplit_eff}.
({\bf Left}): A base tree $\tau_0$ with 3-taxon clade assignments for the tip nodes.
({\bf Right}): The local topologies that can be taken by the subtrees in $C_1$.}
\label{fig:subsplit_eff}
\end{center}
\end{figure}
% \textcolor{blue}{especially when concentrating on tree topologies with high posteriors (see Section \ref{sec:support_estimation} for a more detailed discussion on the efficiency of subsplit representations of trees).}
%\footnote{We defer a more complete description on the efficiency of subsplit representations of trees to Section \ref{sec:support_estimation}.}.
%Although subsplit representations may prevent us from capturing all tree topologies when the choice of subsplits are limited to a pre-selected set as done later in Section \ref{sec:support_estimation}, this also allows us to concentrate on these supported tree topologies and makes it easier for optimization and sampling.
Moreover, with appropriate conditional independence assumptions, the subsplit representations of tree topologies inspire a Bayesian network formulation of distributions over the tree topology space.

\begin{definition}[Subsplit Bayesian Network]
A \emph{subsplit Bayesian network} $\mathcal{B}_\mathcal{X}$ on a leaf label set $\mathcal{X}$ of size $N$ is a Bayesian network whose nodes take on subsplit values or singleton clade values of $\mathcal{X}$, and has the following properties.
\begin{enumerate}
\item The root node of $\mathcal{B}_\mathcal{X}$ takes on subsplits of the entire leaf label set $\mathcal{X}$.
\item $\mathcal{B}_\mathcal{X}$ contains a full and complete binary tree network $\mathcal{B}_\mathcal{X}^\ast$ as a subnetwork.
\item The depth of $\mathcal{B}_\mathcal{X}$ is $N-1$, with the root counting as depth $1$.
\end{enumerate}
\end{definition}

Figure~\ref{fig:sbn} provides an illustration of the structure of general subsplit Bayesian networks (SBNs), together with a concrete example for a leaf set of 4 taxa.
We see that $\mathcal{B}_\mathcal{X}^\ast$ itself is an SBN that is contained in all SBNs, i.e., $\mathcal{B}_\mathcal{X}^\ast$ is the minimum SBN on $\mathcal{X}$.
All SBNs, therefore, share the same set of nodes, and differ only in their conditional independence structures.
%From the above definition, we see that $\mathcal{B}_\mathcal{X}^\ast$ is also an SBN and is contained in all SBNs, i.e., $\mathcal{B}_\mathcal{X}^\ast$ is the minimum SBN on $\mathcal{X}$.
%Although we develop the theory for more general SBNs (illustrated in Figure~\ref{fig:subsplitpair}, Left), we will use this minimum SBN for all of the applications in this paper.
Moreover, there is a natural universal indexing procedure for the nodes of all SBNs based on $\mathcal{B}_\mathcal{X}^\ast$ as follows. First, we denote the root node as $S_1$.
Then, for any $i$, we denote the two children of $S_i$ on $\mathcal{B}_\mathcal{X}^\ast$ as $S_{2i}$ and $S_{2i+1}$, with $S_{2i}$ being on top of $S_{2i+1}$, until a leaf node is reached (the right plot in Figure~\ref{fig:sbn}).
We call the parent nodes in $\mathcal{B}_\mathcal{X}^\ast$ the \emph{natural parent}.
Although this configuration of SBNs would take $2^{N-1}$ nodes in total, only a small portion of these nodes will be used when it comes to tree probability estimation as the others represent trivial local splitting patterns of tree topologies (the dashed gray subgraphs in the middle plots of Figure~\ref{fig:sbn}).
%This, together with conditional probability sharing as introduced in section 3.1 and practical subsplit support estimation in section 4.2, leads to a compact parameterization of SBNs.
In practice, we can use conditional probability sharing in Section \ref{sec:tree_prob_est} and subsplit support estimation in Section \ref{sec:support_estimation} for a compact parameterization of SBNs.

For a rooted tree topology, one can find its unique SBN representation by assigning the subsplit decomposition to the nodes of SBNs according to the same splitting process as before (the middle plots in Figure~\ref{fig:sbn}).
This way, the subsplit for a node is always a bipartition of a subclade in the subsplit for its natural parent.
We refer to this kind of assignments of nodes as \emph{compatible node assignments}.

\begin{definition}[Compatible Node Assignment]\label{def:compatible}
For a non-singleton clade $X$, we say a subsplit $(Y,Z)$ is \emph{compatible} with $X$ iff $Y\cup Z=X$.
For a singleton clade $X$, we say a singleton clade $Y$ is compatible with $X$ iff $Y=X$.
We say a full node assignment $\{S_i=s_i\}_{i\geq 1}$ is \emph{compatible} if for all interior node assignments $s_i$ that are subsplits $(Y_i,Z_i)$, the corresponding assignments of the children nodes $s_{2i}$ and $s_{2i+1}$ are compatible with $Y_i,Z_i$, respectively.
\end{definition}

Clearly, every rooted tree topology corresponds to a unique compatible node assignment.
On the other hand, we can also restore a rooted tree given a compatible node assignment.
Therefore, we have the following lemma.

\begin{lemma}\label{lemma:sbnrep}
There is a one-to-one mapping between rooted tree topologies on $\mathcal{X}$ and compatible node assignments.
\end{lemma}
See a proof of Lemma~\ref{lemma:sbnrep} in Appendix B.
Note that in practice SBNs are often parameterized based on some estimated subsplit supports (see Section \ref{sec:support_estimation} for more details).
Tree topologies that have local splitting patterns that are not covered by those subsplit supports would have zero SBN probabilities.
In other words, compatible node assignments would not necessarily lead to a positive SBN probability of the corresponding tree topologies.
Note that an alternative formulation of an SBN for a choice of subsplit support is naturally described using a directed acyclic graph; see \cite{Jun2023-ib} for more details.

\subsection{Tree Probability Estimation}\label{sec:tree_prob_est}
As Bayesian networks, SBNs can be used to define probabilities for rooted trees based on their corresponding node assignments.
Given the subsplit decomposition $\{s_1,s_2,\ldots\}$ of a rooted tree $\tau$, where $s_1$ is the root subsplit and $\{s_i\}_{i>1}$ are the other subsplits, we define the SBN-induced tree probability of $\tau$ as
\begin{equation}\label{eq:sbn_prob}
p_{\mathrm{sbn}}(T=\tau) = p(S_1=s_1)\prod_{i>1}p(S_i=s_i|S_{\pi_i}=s_{\pi_i})
\end{equation}
where $\pi_i$ denotes the set of indices of the parent nodes of $S_i$ (if we are using the minimal SBN $\mathcal{B}_\mathcal{X}^\ast$ there will only be one parent node).
Note that a naive parameterization of SBNs would require conditional probability distributions (CPDs) for all local conditional models $p(S_i|S_{\pi_i})$ that scales exponentially as the number of taxa $N$ increases.
To reduce the number of parameters and encourage generalization, \emph{conditional probability sharing} is often assumed where the same set of CPDs for parent-child subsplit pairs is shared across the SBN network, regardless of their locations \citep{SBN}.
In practice, this proves effective for a compact parameterization of SBNs (see Section \ref{sec:variational_parameterization} for more details).
%Although Bayesian networks naturally define distributions over all possible assignments, the SBN probabilities defined in~\eqref{eq:sbn_prob} do not necessarily sum to one over all rooted trees as they cover only the compatible assignments, \textcolor{blue}{which is a subset of all possible assignments.}
Traditionally, Bayesian networks define distributions over all possible assignments.
However, in this setting we only want to consider node assignments that form a tree.
To remedy this issue, we therefore impose the following consistency condition for the associated CPDs of SBNs.

\begin{definition}[Consistent Parameterization]
We say the conditional probability $p(S_i|S_{\pi_i})$ is \emph{consistent} if $p(S_i=s_i|S_{\pi_i}=s_{\pi_i}) > 0$ only when the pair $(S_i=s_i,S_{\pi_i}=s_{\pi_i})$ can be extended to a compatible node assignment of all nodes.
We say an SBN is \emph{consistently parameterized} if all the associated conditional probabilities are consistent.
\end{definition}

The consistency condition restricts the support of SBNs to the compatible node assignments, which together with Lemma \ref{lemma:sbnrep}, guarantee that SBNs provide valid distributions over the rooted tree topology space.
\begin{proposition}\label{prop:consistency}
Let $\mathcal{T}_{\mathcal{X}}$ be the set of all rooted trees with leaf labels $\mathcal{X}$. For any consistently parameterized SBN, $\sum_{\tau\in\mathcal{T}_{\mathcal{X}}} p_{\mathrm{sbn}}(T=\tau) = 1$.
\end{proposition}

The proof of Proposition \ref{prop:consistency} is standard and is provided in Appendix B.
Based on Proposition \ref{prop:consistency}, we can design a family of distributions over the rooted tree topology space using the conditional independence structures in SBNs \citep{SBN}, where
$\mathcal{B}_{\mathcal{X}}^\ast$ serves as a base model (Figure \ref{fig:sbn}, Left).
%An illustration can be found in Appendix A.

The SBN framework also generalizes to unrooted trees, which are another popular type of trees in phylogenetic inference that depicts only the relationship between taxa (Figure~\ref{fig:sbnunrooted}, left).
The key idea is to view unrooted trees as rooted trees where the position of the root is treated as a missing variable.
More concretely, unrooted trees can be transformed into (and hence be represented as) a set of rooted trees by placing the virtual roots at the edges.
On the other hand, rooted trees can also be transformed into unrooted trees by removing the roots.
Figure~\ref{fig:sbnunrooted} shows a simple example on a four taxon unrooted tree.
Through this rooting/unrooting operation, the unrooted topology induces an equivalence relation among rooted trees defined as follows.
\begin{definition}[Equivalence Relation]
Let $D$ be the unrooting operator. For any pair of rooted trees $\tau_1$ and $\tau_2$, we say $\tau_1\sim \tau_2$ if $D(\tau_1)$ and $D(\tau_2)$ represent the same unrooted tree, that is, $D(\tau_1)=D(\tau_2)$.
\end{definition}

\begin{figure}[t!]
\begin{center}
\input{figs/bn-unrooted}
\caption{SBNs for unrooted trees.
({\bf Left}): A simple four taxon unrooted tree.
The equivalence class consists of five rooted trees, and each can be obtained by rooting on one of the edges $\{1,2,3,4,5\}$.
({\bf Middle (left)}): Two exemplary rooted trees in the equivalence class when rooting on edges $1$ and $3$.
({\bf Middle (right)}): The corresponding SBN assignments for the two rooted trees.
({\bf Right}): An SBN for the unrooted tree with unobserved root node $S_1$.}
%The structure is similar to those for the rooted trees, except that the root node $S_1$ is missing.}
\label{fig:sbnunrooted}
\end{center}
\end{figure}

It is easy to verify that $\sim$ defines an equivalence relation in the rooted tree space, and every unrooted tree $\tau^\mathrm{u}$ corresponds to an equivalence class $[\tau^\mathrm{u}]=\{\tau: D(\tau)=\tau^\mathrm{u}\}$.
We, therefore, define the SBN probability estimates for unrooted trees as the sum of SBN probabilities for the rooted trees in the equivalence class
\begin{equation}\label{eq:sbnunrooted}
p_{\mathrm{sbn}}(T^\mathrm{u}=\tau^{\mathrm{u}}) = \sum_{\tau\in[\tau^{\mathrm{u}}]}p_{\mathrm{sbn}}(\tau).
\end{equation}

As the equivalence classes naturally partition the entire rooted tree space, \eqref{eq:sbnunrooted} also provides a valid probability distribution in the unrooted tree space.
\begin{proposition}
Let $\mathcal{T}_{\mathcal{X}}^{\mathrm{u}}$ be the set of all unrooted trees with leaf labels $\mathcal{X}$. For any consistently parameterized SBN, $\sum_{\tau^\mathrm{u}\in\mathcal{T}_{\mathcal{X}}^{\mathrm{u}}} p_{\mathrm{sbn}}(T^\mathrm{u}=\tau^{\mathrm{u}}) = 1$.
\end{proposition}
\begin{proof}
\[
\sum_{\tau^\mathrm{u}\in\mathcal{T}_{\mathcal{X}}^{\mathrm{u}}} p_{\mathrm{sbn}}(T^\mathrm{u}=\tau^{\mathrm{u}}) = \sum_{\tau^\mathrm{u}\in\mathcal{T}_{\mathcal{X}}^{\mathrm{u}}}\sum_{\tau\in[\tau^{\mathrm{u}}]}p_{\mathrm{sbn}}(\tau)
= \sum_{\tau\in\mathcal{T}_{\mathcal{X}}}p_{\mathrm{sbn}}(\tau) = 1.
\]
\end{proof}

Note that each rooted tree in the equivalence class corresponds to an edge (i.e., the rooting position) of the unrooted tree, $[\tau^\mathrm{u}]$ therefore can be explicitly represented as $\{\tau^{\stkout{\mathrm{u}}}_e: e\in E(\tau^\mathrm{u})\}$, where $\tau^{\stkout{\mathrm{u}}}_e$ means the resulting rooted tree when the root edge is $e$.
We will use this representation in the sequel.
As each edge corresponds to a root subsplit, \eqref{eq:sbnunrooted} can also be viewed as the marginal SBN probability over the unobserved root node $S_1$ (the right plot in Figure~\ref{fig:sbnunrooted}).

Not only do SBNs provide a family of distributions over the rooted/unrooted tree spaces, they also allow efficient tree probability computation algorithms and tree sampling procedures as we introduce next.
All the involved computation complexity is measured with respect to the size of the tree topology, or equivalently, the number of taxa $N$.

\subsection{Tree Probability Computation}\label{sec:tree_prob_comp}
In the previous section, we have illustrated how to use SBNs for tree probability estimation.
We now show that these SBN-induced tree probabilities can be efficiently computed in linear time $O(N)$ as follows.
For rooted trees, we can use a simple tree traversal to collect the required subsplits in \eqref{eq:sbn_prob} and evaluate the product using the conditional probabilities for the corresponding parent-child subsplit pairs in SBNs.
This leads to a linear time algorithm as both the tree traversal and product evaluation take linear time.
For unrooted trees, \eqref{eq:sbnunrooted} requires computing the SBN probabilities as defined in \eqref{eq:sbn_prob} for all rooted trees in the equivalence class.
A naive approach is to run tree traversals for each of these rooted trees and collect the required subsplits $\cup_{e\in E(\tau^\mathrm{u})}\{s_1^{e,\tau^\mathrm{u}}, s_2^{e,\tau^\mathrm{u}},\ldots\}$, where $\{s_1^{e,\tau^\mathrm{u}}, s_2^{e,\tau^\mathrm{u}},\ldots\}$ is the subsplit decomposition for $\tau^{\stkout{\mathrm{u}}}_e$ (the middle plots in Figure~\ref{fig:sbnunrooted}).
With these subsplits, we can evaluate \eqref{eq:sbnunrooted} as follows
\begin{equation}\label{eq:sbn_unrooted_computation}
p_{\mathrm{sbn}}(\tau^{\mathrm{u}}) = \sum_{e\in E(\tau^{u})}p(S_1=s_1^{e,\tau^\mathrm{u}})\prod_{i>1} p(S_i=s_i^{e,\tau^\mathrm{u}}|S_{\pi_i}=s_{\pi_i}^{e,\tau^\mathrm{u}}).
\end{equation}
Although \eqref{eq:sbn_unrooted_computation} provides a valid computation method, this naive approach has a quadratic time complexity as both the tree traversals and sum-product evaluation takes quadratic time, making it difficult to scale up to large trees.

Fortunately, by leveraging the hierarchical structure of phylogenetic trees, we can design a more efficient linear time algorithm for SBN probability estimation of unrooted trees.
Notice that most of the subsplits for rooted trees in the equivalence class and the intermediate partial products involved in \eqref{eq:sbn_unrooted_computation} are shared due to the hierarchical structure of trees, this allows us to collect all the required subsplits and complete the sum-product evaluation within a two-pass sweep through the unrooted tree, similar to the two-pass algorithm for tree-structured graphical models \citep{JudeaPearl,Schadt1998,Willsky2002}.
Without loss of generality, we use the simplest SBN, $\mathcal{B}_{\mathcal{X}}^\ast$, as an example and describe our two-pass algorithm as below.

\begin{figure}[t!]
\begin{center}
\input{figs/message_passing}
\caption{The two-pass algorithm for computing the SBN probabilities of unrooted trees in \eqref{eq:sbn_unrooted_computation}.
({\bf Left}): The conditional probability for the parent-child subsplit pair appeared in the message updating formulas.
The triangles denote subtrees in a simplified form.
The subtrees on the other side of node $k$ are ignored.
%The associated conditional probability from node $j$ to message $M_{i\rightarrow k}$ is $P_k(j\rightarrow i) = p((X,Y)|(X\cup Y, Z))$.
({\bf Middle}): The postorder (solid black) and preorder (dashed gray) traversals when node $k$ is chosen as the root node (denoted as a square node).
({\bf Right}): The root subsplit $(V\cup Z, X\cup Y)$ and parent-child subsplit pairs $(X,Y)|(V\cup Z, X\cup Y), (V,Z)|(V\cup Z, X\cup Y)$ for edge $e$.
%\textcolor{blue}{The corresponding subsplit probability and parent-child subsplit pair conditional probabilities are the remaining terms in \eqref{eq:sbn_unrooted_computation} that are not accounted for during the two-pass phase.}
}
\label{fig:message_passing}
\end{center}
\end{figure}

First, initialize all pairwise messages for the edges to one: $M_{i\rightarrow j}=M_{j\rightarrow i}=1, \forall (i,j)\in E(\tau^\mathrm{u})$.
Next, choose a node as the root node (e.g., node $k$ in the middle plot in Figure~\ref{fig:message_passing}, we will call this a \emph{virtual root}) and make two passes through the tree as follows.
In the first pass, we traverse the tree in a postorder fashion and update the rootward messages when we visit interior node $i$ (except the root node) as follows
\begin{equation}\label{eq:leaf_to_root_message}
M_{i\rightarrow \pi_i} = \prod_{j\in \operatorname{ch}(i)}P_{\pi_i}(j\rightarrow i)M_{j\rightarrow i}
\end{equation}
where $\operatorname{ch}(i)$ means the child nodes of $i$ and $P_{\pi_i}(j\rightarrow i)$ is the conditional probability for the parent-child subsplit pair representing the local splitting pattern of the child node $j$ at node $i$, given the parent node $\pi_i$.
More specifically, let $(X,Y)$ be the corresponding subsplit at node $j$ looking away from the parent node $i$ and $(X\cup Y, Z)$ be the corresponding subsplit at node $i$ looking away from the parent node $\pi_i$, then $P_{\pi_i}(j\rightarrow i) = p((X,Y)|(X\cup Y, Z))$ (see the left plot in Figure~\ref{fig:message_passing}).
In the second pass, we traverse the tree in a preorder fashion and update the leafward messages when we visit interior node $i$ as follows
%\begin{equation}\label{eq:root_to_leaf_message}
%M_{i\rightarrow k} = \left\{\begin{array}{ll}\prod\limits_{j\in si(k)} P_{k}(j\rightarrow i) M_{j\rightarrow i}, \quad & i \text{ is the root node}\\
%\prod\limits_{j\in si(k)\cup \{\pi_i\}} P_{k}(j\rightarrow i) M_{j\rightarrow i}, \quad &\text{otherwise}
%\end{array}\right.,
%\quad \forall k \in ch(i)
%\end{equation}
%where $si(k)$ means the sister nodes of $k$, i.e., the other child nodes of $i$ except node $k$, and $P_k(j\rightarrow i)$ is defined similarly as $P_{\pi_i}(j\rightarrow i)$ in \eqref{eq:leaf_to_root_message}.
%In fact, if we denote the neighborhood of node $i$ in the unrooted tree as $ne(i)$, the message updates in \eqref{eq:leaf_to_root_message},\eqref{eq:root_to_leaf_message} can be unified with a single equation
\begin{equation}\label{eq:message}
M_{i\rightarrow k} = \prod_{j\in \operatorname{ne}(i)\backslash \{k\}} P_k(j\rightarrow i) M_{j\rightarrow i}, \quad \forall\; k \in \operatorname{ch}(i),
\end{equation}
where $\operatorname{ne}(i)$ means the neighborhood of node $i$ and $P_k(j\rightarrow i)$ is defined similarly as $P_{\pi_i}(j\rightarrow i)$ in \eqref{eq:leaf_to_root_message}.
See an illustration in Figure~\ref{fig:message_passing}.
In both phases of the algorithm, $M_{i\rightarrow k}$ accumulates the contributions to SBN probabilities from the subtree rooted at node $i$ looking away from neighbor $k$.
After the second pass, for each edge $e=(i,j)\in E(\tau^\mathrm{u})$, we can aggregate the contributions to $p_{\mathrm{sbn}}(\tau^\mathrm{u})$ from the subtrees on both sides of the edge as $M_{i\rightarrow j}M_{j\rightarrow i}$.
This, together with the subsplit probability $P_r(e)$ and the parent-child subsplit pair conditional probabilities $P_r(i\rightarrow j), P_r(j\rightarrow i)$ associated with the root edge $e$ (the right plot in Figure~\ref{fig:message_passing})\footnote{These root subsplits and parent-child subsplit pairs associated with the edges turn out to be quite useful. We will revisit them later.}, allow us to compute the SBN probability for each rooted tree $\tau^{\stkout{\mathrm{u}}}_e$ in the equivalence class, and hence compute the marginal SBN probability in \eqref{eq:sbn_unrooted_computation} as
\begin{equation}\label{eq:two_pass}
p_{\mathrm{sbn}}(\tau^\mathrm{u}) = \sum_{e\in E(\tau^\mathrm{u})} p_{\mathrm{sbn}}(\tau^{\stkout{\mathrm{u}}}_e) = \sum_{e=(i,j)\in E(\tau^\mathrm{u})} P_r(e) P_r(i\rightarrow j) P_r(j\rightarrow i) M_{i\rightarrow j}M_{j\rightarrow i}.
\end{equation}
%where $P_r(e)$ represents the root subsplit probability for edge $e$.
Note that the message updates in each iteration of the tree traversals take constant time and the sum-product in \eqref{eq:two_pass} takes linear time, the above two-pass algorithm hence enjoys a linear time complexity.

\subsection{Tree Sampling}
Sampling rooted trees from SBNs is straightforward via ancestral sampling, a method that samples the random variables for the nodes of Bayesian networks in topological order.
More specifically, we start by sampling from the root node; then we sample from the child nodes by conditioning their CPDs to the sampled values of their parent nodes.
We proceed like this until all nodes have been sampled.
According to Lemma \ref{lemma:sbnrep}, the resulting assignment of subsplits uniquely represents a rooted tree topology given that the SBN is consistently parameterized.
Sampling unrooted trees from SBNs is almost identical to sampling rooted trees, but with an additional unrooting operation that transforms the sampled rooted trees into unrooted ones.
As the rooting/unrooting operation defines the correspondence between the unrooted trees and their equivalence classes in the rooted tree space, it is easy to show that this sampling procedure is consistent with the SBN probabilities for unrooted trees defined in \eqref{eq:sbnunrooted}.

\begin{proposition}
Let $T^\mathrm{u}$ be an unrooted tree valued random variable defined implicitly by a procedure that samples and deletes the roots of samples from SBNs (via ancestral sampling).
Then $T^\mathrm{u}$ follows the distribution defined in  \eqref{eq:sbnunrooted}.
\end{proposition}

\begin{proof}
Let $\tau^\mathrm{u}$ be a realization of $T^\mathrm{u}$.
Notice that the set of rooted trees that can be transformed into $\tau^\mathrm{u}$ via the unrooting operation is exactly the equivalence class $[\tau^\mathrm{u}]=\{\tau^{\stkout{\mathrm{u}}}_e:e\in E(\tau^\mathrm{u})\}$.
The probability that $T^\mathrm{u}$ takes value of $\tau^\mathrm{u}$ therefore is
\[
p(T^\mathrm{u}=\tau^\mathrm{u}) = \sum_{\tau\in[\tau^\mathrm{u}]}p_{\mathrm{sbn}}(\tau) = p_{\mathrm{sbn}}(T^\mathrm{u}=\tau^\mathrm{u}).
\]
\end{proof}

Note that ancestral sampling for Bayesian networks allows us to sample from SBNs in linear time $O(N)$ and the unrooting operation takes constant time, sampling rooted/unrooted trees from SBNs therefore can be done in linear time $O(N)$.

\section{Variational Bayesian Phylogenetic Inference}\label{sec:vbpi}
In this section, we present our variational approach for approximate Bayesian phylogenetic inference.
We first introduce the general variational framework with approximating distributions and training objectives.
We then discuss some practical and efficient parameterization strategies for our variational approximations.
Finally, we introduce several efficient computational techniques, including low-variance stochastic gradient estimators for variational training of the model parameters and a linear algorithm for computing the gradient of SBN parameters which is essential for stochastic gradient estimation.

\subsection{A General Framework}\label{sec:framework}
As SBNs can provide a family of distributions over tree topologies with efficient probability estimation and sampling algorithms, they stand out as an essential building block to perform variational inference for phylogenetics.
Given a family of SBN-based approximating distributions $Q_{\bm{\phi}}(\tau)$ over phylogenetic tree topologies, where $\bm{\phi}$ denotes the CPDs of SBNs, we can prescribe our variational approximation as
\[
Q_{\bm{\phi},\bm{\psi}}(\tau,\bm{q}) = Q_{\bm{\phi}}(\tau)Q_{\bm{\psi}}(\bm{q}|\tau)
\]
where $Q_{\bm{\psi}}(\bm{q}|\tau)$ is another family of conditional probability distributions over the branch lengths given the tree topology $\tau$ (see more detailed construction in Section \ref{sec:variational_parameterization}), with $\bm{\psi}$ being the branch length variational parameters.
Variational inference now amounts to finding a candidate of this family that minimizes the Kullback-Leibler (KL) divergence to the exact posterior
\begin{equation}\label{eq:vbpi}
\bm{\phi}^\ast,\bm{\psi}^\ast = \argmin_{\bm{\phi},\bm{\psi}} D_{\mathrm{KL}}\left(Q_{\bm{\phi},\bm{\psi}}(\tau,\bm{q})\|p(\tau,\bm{q}|\bm{Y})\right)
\end{equation}
which is equivalent to maximizing the evidence lower bound (ELBO)
\begin{equation}\label{eq:ELBO_one_sample}
L(\bm{\phi},\bm{\psi}) = \mathbb{E}_{Q_{\bm{\phi},\bm{\psi}}(\tau,\bm{q})}\log\left(\frac{p(\bm{Y}|\tau,\bm{q})p(\tau,\bm{q})}{Q_{\bm{\phi}}(\tau)Q_{\bm{\psi}}(\bm{q}|\tau)}\right)\leq \log p(\bm{Y}).
\end{equation}

While the above single-sample lower bound is commonly used in the VI literature, it is prone to heavily penalize samples that fail to explain the observations.
As a result, the variational approximation tends to focus on the high-probability areas of the exact posterior, leading to simple local variational posteriors that underestimate the uncertainty of the exact posteriors\footnote{This is a behavior known as mode-seeking.}.
Given that phylogenetic posteriors are often multimodal and require strong exploration ability of the inference algorithms to get good performance, we propose an alternative training objective,  a multi-sample lower bound that averages over multiple samples ($K>1$) when estimating the marginal likelihood \citep{IWAE,VIMCO} as follows
\begin{equation}\label{eq:ELBO_multiple_sample}
L^{K}(\bm{\phi},\bm{\psi}) = \mathbb{E}_{Q_{\bm{\phi},\bm{\psi}}(\tau^{1:K},\bm{q}^{1:K})}\log \left(\frac1K\sum_{i=1}^K\frac{p(\bm{Y}|\tau^i,\bm{q}^i)p(\tau^i,\bm{q}^i)}{Q_{\bm{\phi}}(\tau^i)Q_{\bm{\psi}}(\bm{q}^i|\tau^i)}\right)\leq \log p(\bm{Y})
\end{equation}
where $Q_{\bm{\phi},\bm{\psi}}(\tau^{1:K},\bm{q}^{1:K})\equiv \prod_{i=1}^KQ_{\bm{\phi},\bm{\psi}}(\tau^i,\bm{q}^i)$.
Compared to the single-sample lower bound, the multi-sample lower bounds encourage mode-covering behaviors, and hence facilitate the exploration efficiency in the phylogenetic tree space, especially when getting through the less likely territories between the posterior modes.
The multi-sample lower bound gets tighter as the number of samples $K$ increases and converges to the true marginal likelihood in the limit of infinitely many samples \citep{IWAE}.
Note that tighter lower bounds may also deteriorate the training of the variational approximations \citep{Rainforth19}, and so we suggest using a moderate $K$ in practice (see the experiments in Section \ref{sec:experiments_unrooted} and Section \ref{sec:experiments_rooted} for more details).
In the rest of the paper, we use the multi-sample lower bounds and refer to them as lower bounds for short.

\subsection{Subsplit Support Estimation}\label{sec:support_estimation}
As the CPDs of SBNs are associated with all consistent parent-child subsplit pairs, a complete parameterization of SBNs would require an exponentially increasing number of parameters.
In practice, however, we find that the parent-child subsplit pairs observed from trees with relatively high posterior probabilities are a small subset of the entire set of subsplit pairs.
Therefore, if we can find a sufficiently large collection of subsplits from these favorable trees and restrict the \emph{support} of CPDs (where the conditional probabilities are allowed to be nonzero) accordingly, we can construct SBNs that cover most trees in the high posterior area with an affordable number of parameters.
In this section, we explore several easily implemented options for subsplit support estimation on relatively small data sets (that allows us to compare with the exact posteriors), and leave a discussion of other possibilities to the conclusion.

More concretely, we tested four heuristic approaches that provide candidate tree topologies for this kind of subsplit support estimation for SBNs: bootstrapping (BS) based on three commonly used phylogeny reconstruction methods that are (1) neighbor joining (NJ), (2) maximum parsimony (MP) and (3) maximum likelihood (ML); (4) short-run MCMC\@.
The phylogenetic bootstrap performs bootstrapping on the sites of a multiple sequence alignment, which corresponds to the independent-evolution-among-sites assumption embodied in~\eqref{eq:phylo_likelihood}; see \citet{Felsenstein85} for details.
All experiments were conducted on a real data set DS1 \citep{Hedges90} and the results were averaged over 10 replicates for each approach.
Bootstrapping with neighbor joining and maximum parsimony were performed using PAUP$^\ast$ \citep{PAUP}.
Bootstrapping with maximum likelihood was run in UFBoot \citep{UFBOOT}.
All bootstrapping approaches were run with 8,000 replicates.
The classical MCMC Bayesian analyses were done in MrBayes \citep{Ronquist12}, including an extremely long ``golden run'' of 10 billion iterations (sampled every 1,000 iterations with the first $25\%$ discarded as burn-in) to form the ground truth and short runs of one million iterations (sampled every 100 iterations with the first $20\%$ discarded, yielding 8,000 samples) to collect candidate tree topologies.
We assume a uniform prior on the tree topology, an i.i.d. exponential prior with rate parameter 10.0 for the branch lengths and the simple \citet{Jukes69} substitution model in the Bayesian setting.

\begin{figure}[t!]
\begin{center}
\includegraphics[width=\textwidth]{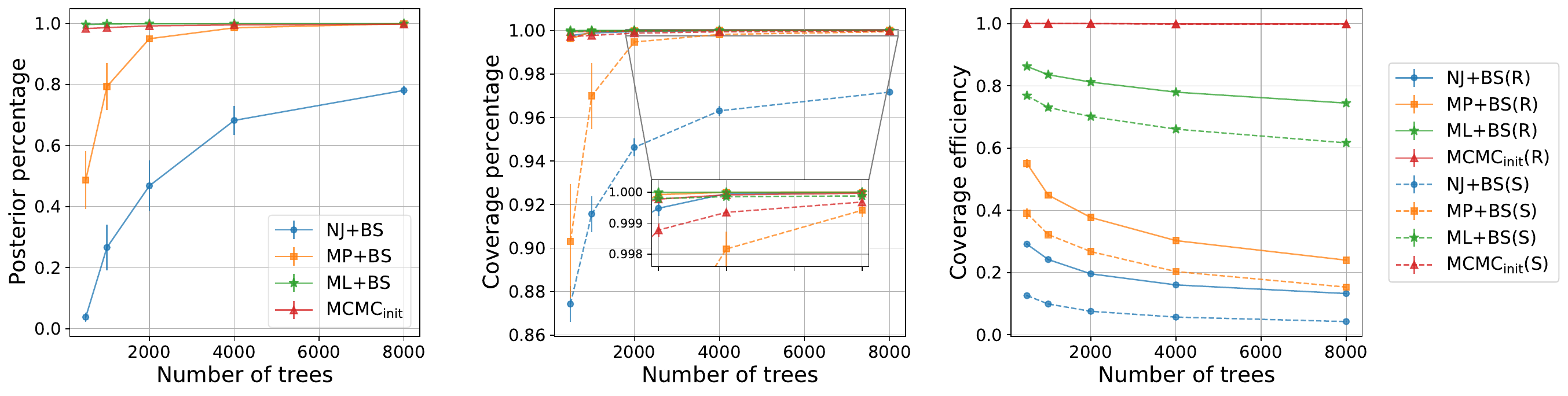}
%\hspace{1.2cm} (a)  \hspace{3.8cm} (b) \hspace{3.55cm} (c)\hspace{2.4cm}
\caption{A comparison on four heuristic approaches for subsplit support estimation. ({\bf Left}): The sum of posterior probabilities of trees covered by SBNs with subsplit support estimated by different methods.
({\bf Middle}): The coverage percentages for root subsplits and subsplit pairs given by different methods.
({\bf Right}): The coverage efficiencies (the portion of root subsplits/subsplit pairs from the candidate trees that also appear in the ground truth collections) for root subsplits and subsplit pairs given by different methods.
The letters R and S in parentheses represent the root subsplit and subsplit pair respectively.
The error bars show one standard error based on 10 independent runs.}
\label{fig:support_estimation}
\end{center}
\end{figure}

We used three different criteria to evaluate the performance of different subsplit support estimation methods.
The left plot in Figure~\ref{fig:support_estimation} shows the covered posterior percentages of the resulting SBNs as a function of the number of trees used to estimate the support.
The covered posterior percentages were estimated by aggregating the posterior probabilities of those sampled trees from the MCMC golden run that have non-zero SBN probabilities, i.e., are covered by the SBNs.
We see that both bootstrapping with the maximum likelihood (ML+BS) and short run MCMC (MCMC$_{\mathrm{init}}$) approaches achieve higher posterior coverage with smaller numbers of trees than the other methods.
Bootstrapping with maximum parsimony (MP+BS) manages to provide comparable posterior coverage, but requires more trees.
Trees obtained from bootstrapping with neighbor joining (NJ+BS) are inadequate to cover the high posterior area.
The middle plot in Figure~\ref{fig:support_estimation} shows the coverage percentages of both root subsplits (solid) and subsplit pairs (dashed) for different methods, where the coverage percentages are defined as the portions of the ground truth root subsplits/subsplit pairs that appear in the candidate trees.
%indicating the superior of likelihood based approaches for Bayesian tasks.
We see that all methods performed quite well in the case of root subsplits.
In the more challenging case of subsplit pairs, bootstrapping with maximum likelihood and short run MCMC, again, outperformed the other methods.
%As large coverage percentages may come with bad rootsplits/subsplit pairs that complexify the training of VBPI,
The right plot in Figure~\ref{fig:support_estimation} investigates the coverage efficiency as a function of the number of candidate trees, where the coverage efficiency was estimated in terms of the portion of root subsplits/subsplit pairs from the candidate trees that also appear in the ground truth collections.
Clearly, we can see that short run MCMC provides the best coverage efficiency (close to 100$\%$), followed by bootstrapping with maximum likelihood, and bootstrapping with maximum parsimony and neighbor joining performed the worst.
Moreover, the coverage efficiency decreases as the number of trees increases for all methods except short run MCMC, so for these other methods this indicates that large coverage percentages may come with large portions of bad root subsplits/subsplit pairs\footnote{This would increase the complexity of the search space and hence makes the training of variational approximations more difficult.}, especially for low quality candidate trees.
%\textcolor{blue}{Since the coverage percentages increase as more candidate trees are collected, this indicates} that large coverage percentages may come with large portions of bad root subsplits/subsplit pairs\footnote{This would increase the complexity of the search space and hence makes the training of variational approximations more difficult.}, especially for low quality candidate trees.
Overall, compared to distance-based methods, likelihood-based approaches are more suitable for subsplit support estimation, in terms of posterior coverage, root subsplit and subsplit pair coverage percentage and efficiency.
We will use bootstrapping with maximum likelihood and short run MCMC for subsplit support estimation in our experiments.
Note that our empirical results here is on DS1, which is a relatively small data set.
We leave a more thorough investigation on how these findings trend as tree sizes increase to future work (see Section \ref{sec:conclusion} for more discussion).

\subsection{Variational Parameterization}\label{sec:variational_parameterization}
Given an appropriate set of candidate trees acquired via the aforementioned approaches, we can simply traverse these trees and collect root subsplits and parent-child subsplit pairs along the way to form the subsplit support.
Now denote the set of root subsplits in the support as $\mathbb{S}_\mathrm{r}$ and the set of parent-child subsplit pairs collected during the tree traversals as $\mathbb{S}_\mathrm{ch|pa}$.
Let $\phi_s$ be the associated parameter for each element $s \in \mathbb{S}_\mathrm{r}\cup \mathbb{S}_\mathrm{ch|pa}$.
We define the CPDs of SBNs according to the following equations
\[
p(S_1=s_1) = \frac{\exp(\phi_{s_1})}{\sum_{s_r\in\mathbb{S}_r}\exp(\phi_{s_r})},\quad p(S_i=s|S_{\pi_i}=t) = \frac{\exp(\phi_{s|t})}{\sum_{s\in\mathbb{S}_{\cdot|t}}\exp(\phi_{s|t})}
\]
where $\mathbb{S}_{\cdot|t}$ denotes the set of child subsplits that are compatible with the greater or smaller subclade in parent subsplit $t$, depending on whether $S_i$ is the top child ($i$ is even) or the bottom child ($i$ is odd) of its parent respectively. Note that this parameterization implicitly assumes conditional probability sharing and is automatically consistent as the root subsplits and parent-child subsplit pairs are collected from tree topologies.

To accommodate the non-negative nature of the branch length parameters in phylogenetic models, we use the Log-normal distribution that has a positive support as our variational distribution.
Furthermore, given a tree topology $\tau$, we assume the branch lengths of $\tau$ are independently distributed.
This leads to a simple approximate conditional density $Q_{\bm{\psi}}(\bm{q}|\tau)$ that takes the following form
\begin{equation}\label{eq:branch_length_distribution}
Q_{\bm{\psi}}(\bm{q}|\tau) = \prod_{e\in E(\tau)} p^{\mathrm{Lognormal}}(q_e \mid \mu(e,\tau),\sigma(e,\tau)).
\end{equation}
Notice that a naive parameterization of \eqref{eq:branch_length_distribution} that assumes independent parameters for different tree topologies would require a large number of parameters, especially when the high posterior areas are diffuse.
To alleviate this issue, we propose to parameterize the branch length distributions over the tree topology space via certain shared local splitting patterns.
More specifically, instead of having separate parameters for each tree topology, we use a shared pool of parameters for these local splitting patterns that maps the tree topologies to their corresponding branch length variational parameters.
We provide two concrete examples as follows.

\subsubsection{Split Based Parameterization}
A simple choice of such local splitting pattern is the \emph{split}, which is a bipartition $(X_1,X_2)$ of the entire leaf label set $\mathcal{X}$ such that $X_1\cup X_2=\mathcal{X},X_1\cap X_2 = \emptyset$.
On a tree topology $\tau$, each edge naturally corresponds to a split, the bipartition that consists of the leaf labels from both sides of the edge.
Denote the corresponding split for edge $e$ of topology $\tau$ as $e/\tau$.
%Since $e/\tau$ can also be viewed as the root subsplit for edge $e$ on $\tau$, it suffices to assign parameters for $\mu,\sigma$ for each root subsplit in $\mathbb{S}_\mathrm{r}$ and parameterize the mean and variance of the corresponding branch length distribution as follows:
Let $\psi_{e/\tau}^\mu$ and $\psi_{e/\tau}^\sigma$ be the associated mean and variance parameters for the corresponding root subsplit in $\mathbb{S}_\mathrm{r}$; the split-based parameterization then takes the following form
\begin{equation}\label{eq:split_parameterization}
\mu(e,\tau) = \psi_{e/\tau}^\mu, \quad \sigma(e,\tau) = \psi_{e/\tau}^\sigma.
\end{equation}

\subsubsection{Primary Subsplit Pair Based Parameterization}\label{sec:psp}
A potential problem with the above split-based parameterization \eqref{eq:split_parameterization} is that it implicitly assumes the branch lengths in different trees have the same distribution as long as they correspond to the same root subsplit, completely ignoring the variation of the branch length distributions among tree topologies.
%To capture this variation, one can use a more sophisticated parameterization that allows other tree-dependent topological structures for the variational parameters $\mu$ and $\sigma$.
To capture this variation, one can use a more sophisticated parameterization where the variational parameters $\mu$ and $\sigma$ may depend on other topological structures.
More specifically, we use the following local structure for each edge in addition to the root subsplits.

\begin{definition}[Primary Subsplit Pair]
Let $e$ be an edge of a phylogenetic tree $\tau$ and say the corresponding root subsplit is $e/\tau = (W,Z)$ where $W\cup Z = \mathcal{X}$.
Assume that at least one of $W$ and $Z$, say $W$, contains more than one leaf label and hence has a subsplit $(W_1,W_2)$.
We call the parent-child subsplit pair $(W_1,W_2)|(W,Z)$ a primary subsplit pair.
\end{definition}

The right plot in Figure~\ref{fig:message_passing} provides an example of primary subsplit pairs for a given edge $e$, where the root subsplit is $(V\cup Z, X\cup Y)$ and the primary subsplit pairs corresponding to the subsplits on either side of $e$ are $(X, Y)|(V\cup Z, X\cup Y)$ and $(V, Z)|(V\cup Z, X\cup Y)$.
Denote the primary subsplit pair(s) for edge $e$ of tree $\tau$ as $e/\!\!/\tau$.
Notice that primary subsplit pairs are also parent-child subsplit pairs, we can assign additional parameters for each primary subsplit pair in $\mathbb{S}_{\mathrm{ch}|\mathrm{pa}}$ and sum all the associated variational parameters for edge $e$ to form the mean and variance of the corresponding branch length distribution (see Figure~\ref{fig:subsplitpair} in Appendix A for an illustration):
\begin{equation}\label{eq:psp_parameterization}
\mu(e,\tau) = \psi_{e/\tau}^\mu + \sum_{s\in e/\!\!/\tau}\psi_s^\mu,\quad \sigma(e,\tau) = \psi_{e/\tau}^\sigma + \sum_{s\in e/\!\!/\tau}\psi_s^\sigma.
\end{equation}
Compared to \eqref{eq:split_parameterization}, \eqref{eq:psp_parameterization} allows contributions from the primary subsplit pairs that contain more topological information and hence could provide more flexible between-tree approximations.
It is worth emphasizing that the above structured parameterizations \eqref{eq:split_parameterization} and \eqref{eq:psp_parameterization} not only reduce the number of required parameters compared to the naive parameterization that assumes independent parameters for different tree topologies, but also enable joint learning of the branch length distributions across tree topologies due to those shared local splitting patterns.

\subsection{Stochastic Gradient Estimators}\label{sec:stochastic_gradient_estimators}
Now that we have constructed our variational distributions for the phylogenetic posteriors, we can find the best approximation by maximizing the lower bound in \eqref{eq:ELBO_multiple_sample} via stochastic gradient ascent.
However, the naive stochastic gradient estimator obtained by direct differentiation, also known as the score function gradient estimator \citep{Rubinstein83, Glynn90, Williams92, Kleijnen96}, is often of very large variance and is impractical for our purpose.
Fortunately, various variance reduction techniques have been introduced in recent years including the control variates \citep{Glynn02, Paisley12, Ranganath13, NVIL, VIMCO} for general latent variables and the reparameterization trick/pathwise derivative \citep{Rubinstein92, Glasserman91, Pflug96, Glasserman13, VAE, Rezende14, Titsias14} for continuous latent variables.
%(\textcolor{blue}{see \citet{Mohamed20} for a more comprehensive overview}).

In what follows, we apply these techniques to different components of our latent variables and derive efficient gradient estimators with much lower variance, respectively.
In addition, we also consider a closely related stable gradient estimator that is based on an alternative variational objective.
More detailed derivations can be found in Appendix C.

\subsubsection{The VIMCO Estimator}
Let $f_{\bm{\phi},{\bm{\psi}}}(\tau, \bm{q}) = \frac{p(\bm{Y}|\tau, \bm{q})p(\tau, \bm{q})}{Q_{\bm{\phi}}(\tau)Q_{\bm{\psi}}(\bm{q}|\tau)}$.
The naive gradient of the lower bound w.r.t. the tree variational parameters $\bm{\phi}$ has the form (see a derivation in Appendix C):
\[
\nabla_{\bm{\phi}} L^K(\bm{\phi}, {\bm{\psi}}) = \mathbb{E}_{Q_{\bm{\phi},{\bm{\psi}}}(\tau^{1:K},\;\bm{q}^{1:K})}\sum_{j=1}^K\left(\hat{L}^K(\bm{\phi},{\bm{\psi}})-\tilde{w}^j\right)\nabla_{\bm{\phi}}\log Q_{\bm{\phi}}(\tau^j)
\]
where $\hat{L}^K(\bm{\phi},{\bm{\psi}}) = \log\left(\frac1K\sum_{i=1}^Kf_{\bm{\phi},{\bm{\psi}}}(\tau^i, \bm{q}^i)\right)$ and $\tilde{w}^j = \frac{f_{\bm{\phi},{\bm{\psi}}}(\tau^j, \bm{q}^j)}{\sum_{i=1}^Kf_{\bm{\phi},{\bm{\psi}}}(\tau^i, \bm{q}^i)}$ are the stochastic lower bound and the self-normalized importance weight, respectively.
While the importance weights are bounded, the stochastic lower bound could be problematic since it does not distinguish the samples according to their fitness for the data and assigns the same ``learning signal'' \citep{NVIL,VIMCO} to all of them.
Moreover, the magnitude of this learning signal can be extremely large, especially in the early stage of training when the samples from the variational distribution explain the data poorly.
%By utilizing the independence between the multiple samples and the regularity of the learning signal, \citet{VIMCO} proposed a localized learning signal strategy that significantly reduces the variance of the gradient estimates as follows
By leveraging these independent and identically distributed multiple samples, \citet{VIMCO} defined a different local learning signal for each sample and proposed a variance-reduced gradient estimate, which leads to the following gradient estimate for the tree variational parameter $\bm{\phi}$:
\begin{equation}\label{eq:vimco}
\nabla_{\bm{\phi}} L^K(\bm{\phi},{\bm{\psi}}) = \mathbb{E}_{Q_{\bm{\phi},{\bm{\psi}}}(\tau^{1:K},\;\bm{q}^{1:K})}\sum_{j=1}^K\left(\hat{L}_{j|-j}^K(\bm{\phi},{\bm{\psi}})-\tilde{w}^j\right)\nabla_{\bm{\phi}}\log Q_{\bm{\phi}}(\tau^j)
\end{equation}
where
\[
\hat{L}_{j|-j}^K(\bm{\phi},{\bm{\psi}}) :=  \hat{L}^K(\bm{\phi},{\bm{\psi}}) - \log\frac1K\left(\sum_{i\neq j}f_{\bm{\phi},{\bm{\psi}}}(\tau^i,\bm{q}^i) + \hat{f}_{\bm{\phi},{\bm{\psi}}}(\tau^{-j},\bm{q}^{-j})\right)
\]
is the per-sample local learning signal, with $\hat{f}_{\bm{\phi},{\bm{\psi}}}(\tau^{-j},\bm{q}^{-j})$ being some estimate of $f_{\bm{\phi},{\bm{\psi}}}(\tau^j,\bm{q}^j)$ for sample $j$ using the rest of the samples (e.g., the geometric mean).
This gives the following VIMCO estimator
\begin{equation}\label{eq:vimcoestimator}
\nabla_{\bm{\phi}} L^K(\bm{\phi},{\bm{\psi}}) \simeq \sum_{j=1}^K\left(\hat{L}_{j|-j}^K(\bm{\phi},{\bm{\psi}})-\tilde{w}^j\right)\nabla_{\bm{\phi}}\log Q_{\bm{\phi}}(\tau^j)
\text{ with }\; \tau^j, \bm{q}^j \mathrel{\overset{\makebox[0pt]{\mbox{\normalfont\tiny\sffamily iid}}}{\sim}} Q_{\bm{\phi},{\bm{\psi}}}(\tau,\;\bm{q}).
\end{equation}

\subsubsection{The Reparameterization Trick}
The above VIMCO estimator also works for the branch length gradient.
However, as the branch lengths are continuous latent variables, we can use an alternative approach, the reparameterization trick, to estimate the gradient.
Note that the Log-normal distribution has a simple reparameterization: $q\sim \mathrm{Lognormal}(\mu,\sigma^2) \Leftrightarrow q = \exp(\mu+\sigma\epsilon),\; \epsilon \sim \mathcal{N}(0,1)$, we can rewrite the lower bound as follows
\[
L^K(\bm{\phi},{\bm{\psi}}) = \mathbb{E}_{Q_{\bm{\phi},\bm{\epsilon}}(\tau^{1:K},\bm{\epsilon}^{1:K})}\log\left(\frac1K\sum_{j=1}^K\frac{p(\bm{Y}|\tau^j,g_{\bm{\psi}}(\bm{\epsilon}^j|\tau^j))p(\tau^j, g_{\bm{\psi}}(\bm{\epsilon}^j|\tau^j))}{Q_{\bm{\phi}}(\tau^j)Q_{\bm{\psi}}(g_{\bm{\psi}}(\bm{\epsilon}^j|\tau^j)|\tau^j)}\right)
\]
where $g_{\bm{\psi}}(\bm{\epsilon}|\tau) = \exp(\bm{\mu}_{\bm{\psi}, \tau}+\bm{\sigma}_{\bm{\psi},\tau}\odot\bm{\epsilon})$.
Then the gradient of the lower bound w.r.t. ${\bm{\psi}}$ is (see Appendix C for derivation)
\begin{equation}\label{eq:reparameterization}
\nabla_{\bm{\psi}} L^K(\bm{\phi},{\bm{\psi}}) = \mathbb{E}_{Q_{\bm{\phi},\bm{\epsilon}}(\tau^{1:K},\bm{\epsilon}^{1:K})} \sum_{j=1}^K \tilde{w}^j \nabla_{\bm{\psi}} \log f_{\bm{\phi},{\bm{\psi}}}(\tau^j,g_{\bm{\psi}}(\bm{\epsilon}^j|\tau^j))
\end{equation}
where $\tilde{w}^j = \frac{f_{\bm{\phi},{\bm{\psi}}}(\tau^j, g_{\bm{\psi}}(\bm{\epsilon}^j|\tau^j))}{\sum_{i=1}^Kf_{\bm{\phi},{\bm{\psi}}}(\tau^i, g_{\bm{\psi}}(\bm{\epsilon}^i|\tau^i))}$ is the same self-normalized importance weight as in equation \eqref{eq:vimco}.
Therefore, we can form the Monte Carlo estimator of the gradient
\begin{equation}\label{eq:reparaestimator}
\nabla_{\bm{\psi}} L^K(\bm{\phi},{\bm{\psi}})  \simeq \sum_{j=1}^K \tilde{w}^j \nabla_{\bm{\psi}} \log f_{\bm{\phi},{\bm{\psi}}}(\tau^j,g_{\bm{\psi}}(\bm{\epsilon}^j|\tau^j))
\text{ with }\; \tau^j \mathrel{\overset{\makebox[0pt]{\mbox{\normalfont\tiny\sffamily iid}}}{\sim}} Q_{\bm{\phi}}(\tau),\; \bm{\epsilon}^j \mathrel{\overset{\makebox[0pt]{\mbox{\normalfont\tiny\sffamily iid}}}{\sim}} p_{\bm{\epsilon}}(\bm{\epsilon}).
\end{equation}

\subsubsection{Self-normalized Importance Sampling Estimator}
In addition to the standard variational formulation \eqref{eq:vbpi}, one can reformulate the optimization problem by minimizing the reversed KL divergence, which is equivalent to maximizing the likelihood of the variational approximation
\begin{equation}\label{eq:rws}
Q_{{\bm{\phi}}^\ast,{\bm{\psi}}^\ast}(\tau,{\bm{q}}), \text{ where } {\bm{\phi}}^\ast,{\bm{\psi}}^\ast = \argmax_{{\bm{\phi}}, {\bm{\psi}}} \tilde{L}({\bm{\phi}},{\bm{\psi}}),\;
\tilde{L}({\bm{\phi}},{\bm{\psi}}) = \mathbb{E}_{p(\tau,{\bm{q}}|{\bm{Y}})}\log Q_{{\bm{\phi}},{\bm{\psi}}}(\tau, {\bm{q}}).
\end{equation}
We can use an importance sampling estimator to compute the gradient of the objective
\begin{align}
\nabla_{{\bm{\phi}}} \tilde{L}({\bm{\phi}},{\bm{\psi}}) &= \mathbb{E}_{p(\tau,{\bm{q}}|{\bm{Y}})}\nabla_{{\bm{\phi}}}\log Q_{{\bm{\phi}},{\bm{\psi}}}(\tau, {\bm{q}}) = \frac1{p({\bm{Y}})}\mathbb{E}_{Q_{{\bm{\phi}},{\bm{\psi}}}(\tau, {\bm{q}})} \frac{p({\bm{Y}}|\tau,{\bm{q}})p(\tau,{\bm{q}})}{Q_{{\bm{\phi}}}(\tau)Q_{{\bm{\psi}}}({\bm{q}}|\tau)}\nabla_{{\bm{\phi}}}\log Q_{{\bm{\phi}}}(\tau) \nonumber \\
& \simeq \sum_{j=1}^K\tilde{w}^j\nabla_{{\bm{\phi}}}\log Q_{{\bm{\phi}}}(\tau^j) \;\text{ with }\; \tau^j, \bm{q}^j \mathrel{\overset{\makebox[0pt]{\mbox{\normalfont\tiny\sffamily iid}}}{\sim}} Q_{{\bm{\phi}},{\bm{\psi}}}(\tau,\;{\bm{q}})\label{eq:rwsgrad}
\end{align}
with the same importance weights $\tilde{w}^j$ as in \eqref{eq:vimco} (a detailed derivation can be found in Appendix C).
This can be viewed as a multi-sample generalization of the wake-sleep algorithm \citep{WS} and was first used in the \emph{reweighted wake-sleep} algorithm \citep{RWS} for training deep generative models.
We therefore call the gradient estimator in \eqref{eq:rwsgrad} the RWS estimator.
Like the VIMCO estimator, the RWS estimator also provides gradients for the branch lengths.
However, we find in practice the gradient estimator in \eqref{eq:reparaestimator} that uses the reparameterization trick is more useful and often leads to faster convergence, although it uses a different optimization objective.
A better understanding of this phenomenon would be an interesting subject of future research.

All stochastic gradient estimators introduced above can be used in conjunction with stochastic optimization methods such as stochastic gradient ascent (SGA) \citep{SGA} or some of its adaptive variants (e.g. Adam \citep{ADAM}) to maximize the lower bounds. See algorithm \ref{alg:vbpi} below for a basic variational Bayesian phylogenetic inference (VBPI) approach.
In the experiments, we will use the VIMCO estimator and RWS estimator for the tree topology parameter $\bm{\phi}$ and the reparameterization trick for the branch length parameter $\bm{\psi}$.

\begin{algorithm*}[t!]
\caption{The variational Bayesian phylogenetic inference (VBPI) algorithm. }
\begin{algorithmic}[1]
\State $\bm{\phi}, \bm{\psi} \gets$ Initialize parameters
\While{not converged}
  \State $\tau^1, \ldots, \tau^K \gets $ Random samples from the current approximating tree distribution $Q_{\bm{\phi}}(\tau)$
  \State $\bm{\epsilon}^1, \ldots, \bm{\epsilon}^K \gets$ Random samples from the multivariate standard normal distribution $\mathcal{N}(\bm{0}, \bm{I})$
  \State $\bm{g} \gets \nabla_{\bm{\phi},\bm{\psi}} L^K(\bm{\phi},\bm{\psi}; \tau^{1:K}, \bm{\epsilon}^{1:K})$ (Use any gradient estimator from Section \ref{sec:stochastic_gradient_estimators})
  \State $\bm{\phi}, \bm{\psi} \gets $ Update parameters using gradients $\bm{g}$ (e.g. SGA)
\EndWhile
\State \Return $\bm{\phi}, \bm{\psi}$
\end{algorithmic}
\label{alg:vbpi}
\end{algorithm*}

\subsection{Efficient Computation of SBN Gradients}\label{sec:sbn_grad_comp}
In this section, we describe some efficient computational methods for $\nabla_{\bm{\phi}}\log Q_{\bm{\phi}}(\tau)$, which is an essential ingredient for the stochastic gradient estimators derived in Section \ref{sec:stochastic_gradient_estimators}.
Similarly to the tree probability computation, the computation of $\nabla_{\bm{\phi}}\log Q_{\bm{\phi}}(\tau)$ for rooted trees is straightforward and can be done in linear time as follows:
\begin{align}\label{eq:sbn_grad_rooted}
\nabla_{\bm{\phi}}\log Q_{\bm{\phi}}(\tau) &= \nabla_{\bm{\phi}}\log p(S_1=s_1) + \sum_{i>1}\nabla_{\bm{\phi}}\log p(S_i=s_i|S_{\pi_i}=s_{\pi_i})\nonumber\\
&=\sum_{s\in\mathrm{pcsp}(\tau)} \nabla_{\bm{\phi}}\log p(s)
\end{align}
where $\{s_i\}_{i\geq 1}$ is the subsplit decomposition of $\tau$ and $\mathrm{pcsp}(\tau)=\{s_1\}\cup \{s_i|s_{\pi_i}\}_{i>1}$ is the set of the root subsplit and parent-child subsplit pairs of $\tau$.

For unrooted trees, we have
\begin{align}\label{eq:sbn_grad_naive}
\nabla_{\bm{\phi}}\log Q_{\bm{\phi}}(\tau^\mathrm{u})  &= \nabla_{\bm{\phi}} \log \sum_{\tau\in[\tau^\mathrm{u}]}p_{\mathrm{sbn}}(\tau) =\sum_{\tau\in [\tau^\mathrm{u}]}\frac{\nabla_{\bm{\phi}}p_{\mathrm{sbn}}(\tau)}{\sum_{\tau\in[\tau^\mathrm{u}]}p_{\mathrm{sbn}}(\tau)}\nonumber\\
& = \sum_{\tau\in [\tau^\mathrm{u}]}\frac{p_{\mathrm{sbn}}(\tau)}{\sum_{\tau\in[\tau^\mathrm{u}]}p_{\mathrm{sbn}}(\tau)}\nabla_{\bm{\phi}}\log p_{\mathrm{sbn}}(\tau)\nonumber\\
& = \sum_{e\in E(\tau^\mathrm{u})}w_e\nabla_{\bm{\phi}}\log p_{\mathrm{sbn}}(\tau^{\stkout{\mathrm{u}}}_e)
\end{align}
where $w_e=\frac{p_{\mathrm{sbn}}(\tau^{\stkout{\mathrm{u}}}_e)}{\sum_{e\in E(\tau^\mathrm{u})}p_{\mathrm{sbn}}(\tau^{\stkout{\mathrm{u}}}_e)}$ is the rooting probability for edge $e$ of $\tau^\mathrm{u}$ induced by the SBN.
Thanks to the efficient two-pass algorithm introduced in Section \ref{sec:tree_prob_comp}, we can compute $\{p_{\mathrm{sbn}}(\tau^{\stkout{\mathrm{u}}}_e):e\in E(\tau^\mathrm{u})\}$ and hence the edge rooting probabilities $\{w_e:e\in E(\tau^\mathrm{u})\}$ in linear time.
However, a naive implementation of \eqref{eq:sbn_grad_naive} still requires quadratic time since $\nabla_{\bm{\phi}}\log p_{\mathrm{sbn}}(\tau^{\stkout{\mathrm{u}}}_e)$ for each edge takes linear time.

Fortunately, by properly utilizing the structure of the tree topology, \eqref{eq:sbn_grad_naive} can also be computed in linear time.
A key observation is that if we expand $\nabla_{\bm{\phi}}\log p_{\mathrm{sbn}}(\tau^{\stkout{\mathrm{u}}}_e)$ as in \eqref{eq:sbn_grad_rooted} for each edge $e$, most of the gradient terms $\nabla_{\bm{\phi}}\log p(s)$ for the parent-child subsplit pairs will be shared across different edges.
Therefore, we can switch the order of summation and accumulate the coefficients for the gradient terms first as follows:
\begin{align}\label{eq:sbn_grad_rearranged}
\nabla_{\bm{\phi}}\log Q_{\bm{\phi}}(\tau^\mathrm{u}) &= \sum_{e\in E(\tau^\mathrm{u})}w_e\sum_{s\in\mathrm{pcsp}(\tau^{\stkout{\mathrm{u}}}_e)}\nabla_{\bm{\phi}}\log p(s)\nonumber\\
&= \sum_{s\in \cup_{e\in E(\tau^\mathrm{u})}\mathrm{pcsp}(\tau^{\stkout{\mathrm{u}}}_e)}\nabla_{\bm{\phi}}\log p(s)\sum_{e\in E(\tau^\mathrm{u})} w_e\delta_{s\in\mathrm{pcsp}(\tau^{\stkout{\mathrm{u}}}_e)}.
\end{align}
Let $w_s=\sum_{e\in E(\tau^\mathrm{u})} w_e\delta_{s\in\mathrm{pcsp}(\tau^{\stkout{\mathrm{u}}}_e)}$.
For the root subsplits and primary subsplit pairs, there is a unique edge $e$ such that $s\in\mathrm{pcsp}(\tau^{\stkout{\mathrm{u}}}_e)$, and so $w_s=w_e$.
For other parent-child subsplit pairs $s$, there is a corresponding subtree of $\tau^\mathrm{u}$ that is rooted at the parent node and looks toward the child node and we denote it as $\tau^\mathrm{u}|s$ (see the left plot in Figure~\ref{fig:message_passing}).
It turns out that $s$ is a parent-child subsplit pair of $\tau^{\stkout{\mathrm{u}}}_e$ if and only if $e$ is not an edge of the subtree $\tau^\mathrm{u}|s$.
Therefore, $w_s$ can be obtained from the cumulative probability of the root distribution along the edges of the tree topology which can be computed via a postorder tree traversal (with a virtual root as in Section \ref{sec:tree_prob_comp}) that accumulates the edge rooting probabilities.
Let $\bar{w}_e$ be the cumulative probability at edge $e$.
Depending on whether the virtual root $r^{\mathrm{v}}$ is in the subtree $\tau^\mathrm{u}|s$ or not, we can compute $w_s$ as follows:
\begin{equation}
w_s = \left\{
\begin{array}{ll}
\bar{w}_e, & r^{\mathrm{v}} \text{ is in } \tau^\mathrm{u}|s\\
1-\bar{w}_e + w_e, & \text{ otherwise}.
\end{array}
\right.
\end{equation}
Finally, for each edge of the tree topology, we can collect the corresponding root subsplit, primary subsplit pairs and regular parent-child subsplit pairs with all possible orientations induced by the relative position of the edge w.r.t the virtual root, and hence complete the summation over the entire set of $\cup_{e\in E(\tau^\mathrm{u})}\mathrm{pcsp}(\tau^{\stkout{\mathrm{u}}}_e)$ within a simple tree traversal.
This concludes our linear time algorithm for the gradient computation of SBN parameters for unrooted trees.

\section{Experiments on Unrooted Phylogenies}\label{sec:experiments_unrooted}
Throughout this section we evaluate the effectiveness and efficiency of our variational framework for Bayesian phylogenetic inference on unrooted phylogenetic trees, with comparisons to results that were obtained from intensive MCMC runs.
The simplest SBN (the one with a full and complete binary tree structure) was used for the phylogenetic tree topology variational distribution; we have found it to provide sufficiently accurate approximation.
%The subsplit supports were estimated from either short MCMC runs using MrBayes \citep{Ronquist12} or BEAST \citep{Suchard18}, or ultrafast maximum likelihood phylogenetic bootstrap using UFBoot \citep{UFBOOT}.
The subsplit supports were estimated from either short MCMC runs using MrBayes \citep{Ronquist12} or ultrafast maximum likelihood phylogenetic bootstrap using UFBoot \citep{UFBOOT}.
We investigate the performance of the VIMCO estimator and the RWS estimator with different variational parameterizations for the branch length distributions, while varying the number of samples in the training objective to see how these affect the quality of the variational approximations.
All models were implemented in PyTorch \citep{Pytorch} with the Adam optimizer \citep{ADAM}.
Results were collected after 200,000 parameter updates.
The reported KL divergences are over the discrete collection of phylogenetic tree topologies, from the estimated posteriors $\hat{Q}(\tau)$ (e.g., SBN distributions from VBPI or simple relative frequencies from standard MCMC) to the ground truth posterior $p(\tau|\bm{Y})$, and they are computed as follows
\begin{equation}\label{eq:kl_topology}
D_{\mathrm{KL}}\left(p(\tau|\bm{Y})\|\hat{Q}(\tau)\right) = \sum_{\tau} p(\tau|\bm{Y}) \log\left(\frac{p(\tau|\bm{Y})}{\max\{\hat{Q}(\tau),\xi\}}\right),
\end{equation}
where $\xi$ is a small positive number that ensures that the denominator is positive (We used the machine precision for $\xi$ in our experiments. See other choices of $\xi$ in Appendix D).
A low KL divergence means a high quality approximation to the exact posterior of the tree topologies.
The code is available at \url{https://github.com/zcrabbit/vbpi-torch}.
%\subsection{Variational Bayesian Inference for Unrooted Phylogenies}\label{sec:vbpi_unrooted}

%We evaluate the proposed variational Bayesian phylogenetic inference (VBPI) algorithms at estimating unrooted phylogenetic tree posteriors on 8 real datasets commonly used to benchmark phylogenetic MCMC methods \citep{Lakner08, Hhna2012-pm,Larget2013-et, Whidden2015-eq}.
We test the proposed variational Bayesian phylogenetic inference (VBPI) algorithms on 8 real data sets that are commonly used to benchmark phylogenetic MCMC methods \citep{Lakner08, Hhna2012-pm,Larget2013-et, Whidden2015-eq}.
We concentrate on the most challenging part of the phylogenetic model: joint learning of the tree topologies and the branch lengths.
We assume a uniform prior on the tree topology, an i.i.d. exponential prior with rate parameter 10.0 for the branch lengths and the simple \citet{Jukes69} substitution model.
We consider two different variational parameterizations for the branch length distributions as introduced in Section \ref{sec:variational_parameterization}. In the first case, we use the simple split-based parameterization that assigns parameters to the splits associated with the edges of the trees. In the second case, we assign additional parameters for the primary subsplit pairs (PSP) to better capture the between-tree variation.
We formed our ground truth posterior from an extremely long MCMC run of $10$ billion iterations (sampled each 1,000 iterations with the first $25\%$ discarded as burn-in) using MrBayes \citep{Ronquist12}, and gathered the subsplit supports from trees generated by 10 replicates of 10,000 runs of the ultrafast maximum likelihood bootstrap \citep{UFBOOT}.
Following \citet{NF}, we use a simple annealed version of the lower bound which was found to provide better results. The modified bound is:
\[
L^K_{\beta_t}({\bm{\phi}},{\bm{\psi}}) = \mathbb{E}_{Q_{{\bm{\phi}},{\bm{\psi}}}(\tau^{1:K},\; \bm{q}^{1:K})}\log\left(\frac1K\sum_{i=1}^K\frac{[p(\bm{Y}|\tau^i, \bm{q}^i)]^{\beta_t} p(\tau^i, \bm{q}^i)}{Q_{{\bm{\phi}}}(\tau^i)Q_{\bm{\psi}}(\bm{q}^i|\tau^i)}\right)
\]
where $\beta_t = \min(1, 0.001+ t/100,000)$ is an inverse temperature schedule that goes from $0.001$ to $1$ after $99,900$ iterations. We trained the variational approximations with VIMCO and RWS using 10 and 20 sample objectives and used Adam with the default learning rate $0.001$ for both methods.
We used an exponential learning rate schedule that decays the learning rate by a factor of 0.75 every 20,000 iterations.
To encourage exploration in the tree topology space at the beginning, we initialize the variational parameters $\bm{\phi}$ for SBNs to zero (which leads to uniformly distributed CPDs).
The same settings are applied to all data sets.

\subsection{Tree Topology Posterior Approximation}
In this section, we examine the performance of VBPI for tree topology posterior approximation.
The left and middle plots in Figure~\ref{fig:DS1} show the resulting KL divergence to the ground truth on DS1 as a function of the number of parameter updates (the right plot will be described below in Section \ref{sec:lb_mle}).
The results for methods that adopt the simple split-based parameterization of variational branch length distributions are shown in the left plot.
We see that the performance of all methods improves significantly as the number of samples $K$ increases.
The middle plot shows the results using PSP for variational parameterization, which clearly indicates that a better modeling of between-tree variation of the branch length distributions is beneficial for all methods with different stochastic gradient estimators and numbers of samples $K$.
Specifically, PSP enables more flexible branch length distributions across tree topologies which makes the learning task much easier, as evidenced by the considerably smaller gaps between the methods.
In both cases, we see that RWS learns faster at the beginning while VIMCO performs better in the end.
To benchmark the learning efficiency of VBPI, we also compare to MrBayes 3.2.5 \citep{Ronquist12}, a standard MCMC implementation for phylogenetic analysis on unrooted trees.
We ran MrBayes with 4 chains and 10 runs for two million iterations, sampling every 100 iterations.
For each run, we computed the KL divergence to the ground truth every 50,000 iterations with the first $25\%$ discarded as burn-in.\footnote{Note that the ``burn-in" here is only to trim the early samples and is different from the much longer standard burn-in for MCMC runs that ensures convergence.}
For a relatively fair comparison (in terms of the computational complexity of likelihood evaluations),
we compare 10 (i.e.\ 2$\cdot$20/4) times the number of MCMC iterations with the number of 20-sample objective VBPI iterations.\footnote{The extra factor of 2/4 is because the likelihood and the gradient can be computed together in twice the time of a likelihood computation (see \citet{Schadt1998}) and we run MCMC with 4 chains.}
As most MCMC moves can leverage likelihood vector caching for a more efficient implementation, we also provide runtime comparisons where the runtime of VBPI is lower bounded by the likelihood and gradient computation in Appendix D.
Although MCMC converges faster at the start, we see that VBPI methods (especially those with PSP) quickly surpass MCMC and arrive at good approximations with much fewer iterations.
This is because VBPI updates the approximate distributions of tree topologies (e.g., SBNs) and branch lengths on the fly, which in turn allows guided exploration in the phylogenetic tree space.
As a commonly used summarization of tree samples, we also compare the majority-rule consensus trees obtained from VBPI and the ground truth MCMC run.
The results on DS1 are presented in Figure~\ref{fig:consensus} in Appendix F.
We see that VBPI and MCMC provide the same majority-rule consensus tree.

\begin{figure}[t!]
\begin{center}
\includegraphics[width=\textwidth]{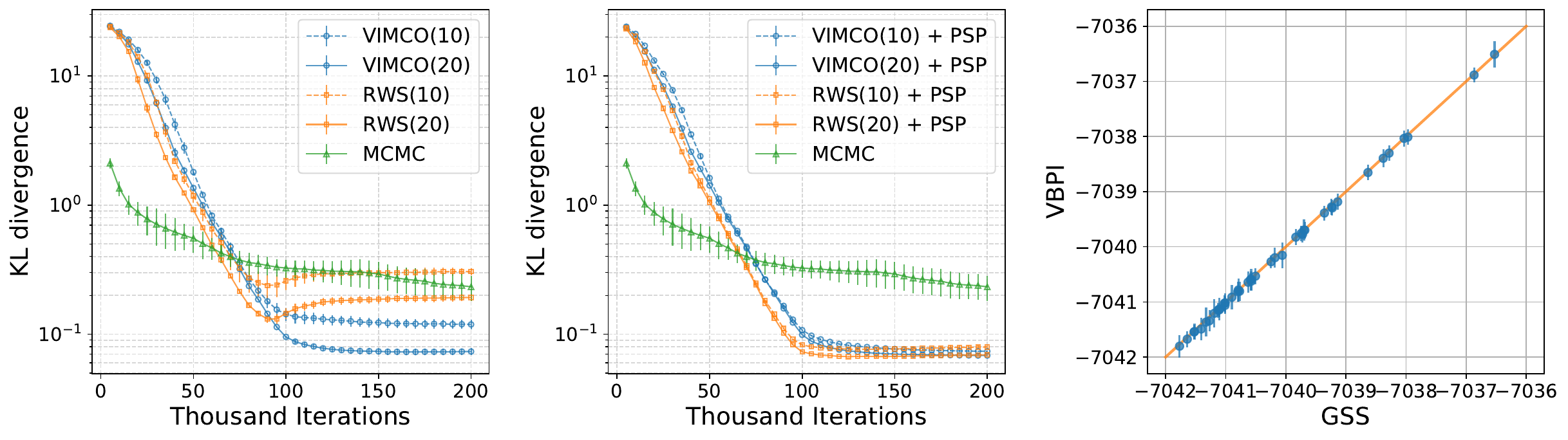}
%\hspace{2.0cm} (a)  \hspace{4.1cm} (b) \hspace{4.3cm} (c)\hspace{1.5cm}
\caption{Performance on DS1. ({\bf Left}): KL divergence for methods that use the simple split-based parameterization for the branch length distributions. ({\bf Middle}): KL divergence for methods that use PSP. ({\bf Right}): Per-tree marginal likelihood estimation (in nats): VBPI vs the Generalized stepping-stone estimator (GSS, described in the text). The number in parentheses specifies the number of samples used in the training objective. KL results are obtained from 10 independent runs. The marginal likelihood estimates for VBPI on each tree were obtained using 1,000 samples and the error bar shows one standard deviation over 100 independent runs.
}\label{fig:DS1}
\end{center}
\vspace{-0.2in}
\end{figure}

\begin{table}[t!]
\caption{Variational lower bound and marginal likelihood estimation of different methods across 8 benchmark data sets for Bayesian phylogenetic inference.
The marginal likelihood estimates of all variational methods were obtained via importance sampling using 1,000 samples, and the results (in units of nats) were averaged over 100 independent runs with standard deviation in parentheses.
We ran the annealed importance sampling stepping-stone estimator (SS)~\citep{SS} in MrBayes using default settings, with 4 chains for 10,000,000 iterations and sampled every 100 iterations.
The results were averaged over 10 independent runs with standard deviation in parentheses.
SS is not a variational method thus there are no entries for it in the upper part of the table.
}\label{tab:realdata}
\scalebox{0.72}{
\begin{threeparttable}
\begin{tabular}{lllllll}
\toprule
data sets & (\#Taxa,\#Sites) & VIMCO($\boldsymbol{10}$) & VIMCO($\boldsymbol{20}$) & VIMCO($\boldsymbol{10}$)$+$PSP & VIMCO($\boldsymbol{20}$)$+$PSP & SS\\
\midrule
&& \multicolumn{5}{c}{Variational Lower Bound (nats)}\\
\cmidrule{3-7}\\[-0.7em]
DS1 & (27, 1949)& -7112.39(1.44) & -7112.60(2.10) & -7111.17(0.99) & -7111.57(1.10) & -- \\
DS2 & (29, 2520)& -26369.89(0.95) & -26369.80(0.94) & -26369.44(0.55) & -26369.57(0.75) & -- \\
DS3 & (36, 1812)& -33736.81(0.34) & -33736.94(0.45) & -33736.70(0.36) & -33736.82(0.37) & --\\
DS4 & (41, 1137)& -13333.00(0.77) & -13333.25(0.84) & -13332.32(0.43) & -13332.70(0.45) &-- \\
%\hiderowcolors
DS5 & (50, 378) &  -8219.00(0.24) & -8129.30(0.32) & -8218.44(0.19) & -8218.74(0.21) &--\\
DS6 & (50, 1133)& -6729.78(0.74) & -6730.40(0.77) & -6729.21(0.36) & -6730.24(0.97) &-- \\
DS7 & (59, 1824)& -37335.44(0.13) & -37335.81(0.14) & -37335.25(0.12) & -37335.52(0.15) &-- \\
DS8 & (64, 1008)& -8659.45(0.65) & -8661.09(1.21) & -8655.55(0.40) & -8655.96(0.52) &-- \\
\midrule
&& \multicolumn{5}{c}{Marginal Likelihood (nats)}\\
\cmidrule{3-7}\\[-0.7em]
%\headrow
%&&\thead{VIMCO($\boldsymbol{10}$)} & \thead{VIMCO($\boldsymbol{20}$)} & \thead{VIMCO($\boldsymbol{10}$)$+$PSP} & \thead{VIMCO($\boldsymbol{20}$)$+$PSP} & \thead{SS}\\
DS1 & (27, 1949)& -7108.48(0.26) & -7108.46(0.23) & -7108.41(0.17) & -7108.41(0.16) & -7108.42(0.18)\\
DS2 & (29, 2520)& -26367.73(0.10) & -26367.74(0.09) & -26367.73(0.09) & -26367.73(0.09) &  -26367.57(0.48)\\
DS3 & (36, 1812)& -33735.13(0.12) & -33735.12(0.12) & -33735.12(0.13) & -33735.12(0.11) & -33735.44(0.50)\\
DS4 & (41, 1137)& -13330.02(0.31) & -13330.01(0.26) & -13329.96(0.24) & -13329.95(0.20) & -13330.06(0.54)\\
%\hiderowcolors
DS5 & (50, 378)  & -8214.78(0.56) & -8214.77(0.53) & -8214.67(0.44) & -8214.64(0.46) & -8214.51(0.28)\\
DS6 & (50, 1133)& -6724.55(0.56) & -6724.54(0.52) & -6724.45(0.43) & -6724.41(0.43) & -6724.07(0.86) \\
DS7 & (59, 1824)& -37332.16(0.51) & -37332.16(0.38) & -37332.05(0.33) & -37332.03(0.31) & -37332.76(2.42)\\
DS8 & (64, 1008)& -8652.93(1.09) & -8651.33(0.87) & -8650.66(0.56) & -8650.64(0.53) & -8649.88(1.75)\\
\bottomrule  % Please only put a hline at the end of the table
\end{tabular}

%\begin{tablenotes}
%\item VLB, variational lower bound; ML, marginal likelihood.
%\end{tablenotes}
\end{threeparttable}
}
\end{table}

\subsection{Lower Bound and Marginal Likelihood Estimation}\label{sec:lb_mle}
To evaluate the overall approximation accuracy of VBPI for phylogenetic posteriors, we report the standard single-sample evidence lower bounds (ELBO) for variational approximations trained via the VIMCO gradient estimator in Table~\ref{tab:realdata}\footnote{Note that this is different from the multi-sample lower bounds that we used as the training objectives.}.
As before, we see that the ELBOs are all significantly improved with the PSP parameterization, in terms of both the estimated means and variances.
Although having a larger number of samples $K$ in the training objective has been shown to facilitate topological exploration for better tree topology posterior estimation, we see that it does not necessarily improve the overall posterior approximation in terms of the ELBO (as observed and analyzed in \citet{Rainforth19}).
However, as we see next, tighter lower bounds indeed bias the variational approximations towards providing better marginal likelihood estimates, especially for the more challenging learning without the PSP parameterization.

The variational approximations provided by VBPI can be readily used to perform importance sampling for phylogenetic inference (see more details in Appendix G).
Thanks to the structured amortization for the branch length distributions, VBPI can provide variational approximations on branch lengths for any tree topology that is covered by the subsplit support.
We, therefore, can use them to estimate the marginal likelihood of tree topologies via importance sampling.
The right plot in Figure~\ref{fig:DS1} compares VBPI using VIMCO with 20-sample objective and PSP to the state-of-the-art generalized stepping-stone (GSS) \citep{GSS} algorithm for estimating the marginal likelihood of trees in the $95\%$ credible set of DS1.
For GSS, we used 50 power posteriors and for each power posterior we ran 1,000,000 MCMC iterations, sampling every 1,000 iterations with the first $10\%$ discarded as burn-in.
The reference distribution \citep[a distribution that matches the posterior more closely than the prior distribution; ][]{GSS} for GSS was obtained from an independent Gamma approximation using the maximum a posteriori estimate.
The GSS estimates were computed using the phylogenetic package \texttt{physher} \citep{Fourment2014} (\url{https://github.com/4ment/physher}).
We see that the estimates given by VBPI closely match GSS estimates (the scatter points are close to the orange diagonal line) and have small variance in general.
Similarly, we can use VBPI to estimate the marginal likelihood of the data (i.e., the model evidence) via importance sampling.
Table~\ref{tab:realdata} (the lower part) also shows the marginal likelihood estimates using different VIMCO approximations and one of the state-of-the-art methods, the stepping-stone (SS) algorithm \citep{SS}, which unlike GSS, is able to also integrate out tree topologies.
For each data set, all methods provide estimates for the same marginal likelihood, with better approximation leading to lower variance.
We see that VBPI using 1,000 samples is already competitive with SS using 100,000 samples and provides estimates with much less variance (hence more reproducible and reliable).
Increasing the number of samples $K$ is helpful for providing stabler marginal likelihood estimates, especially for learning without the PSP parameterization.
With the extra flexibility enabled by PSP, the demand for larger number of samples used in the training objective has been greatly alleviated, making it possible to achieve high quality variational approximations with less samples.

\subsection{Benchmarking Scalability Using Influenza Data}\label{sec:curated_data_benchmark}

\begin{figure}[t!]
\begin{center}
\includegraphics[width=\textwidth]{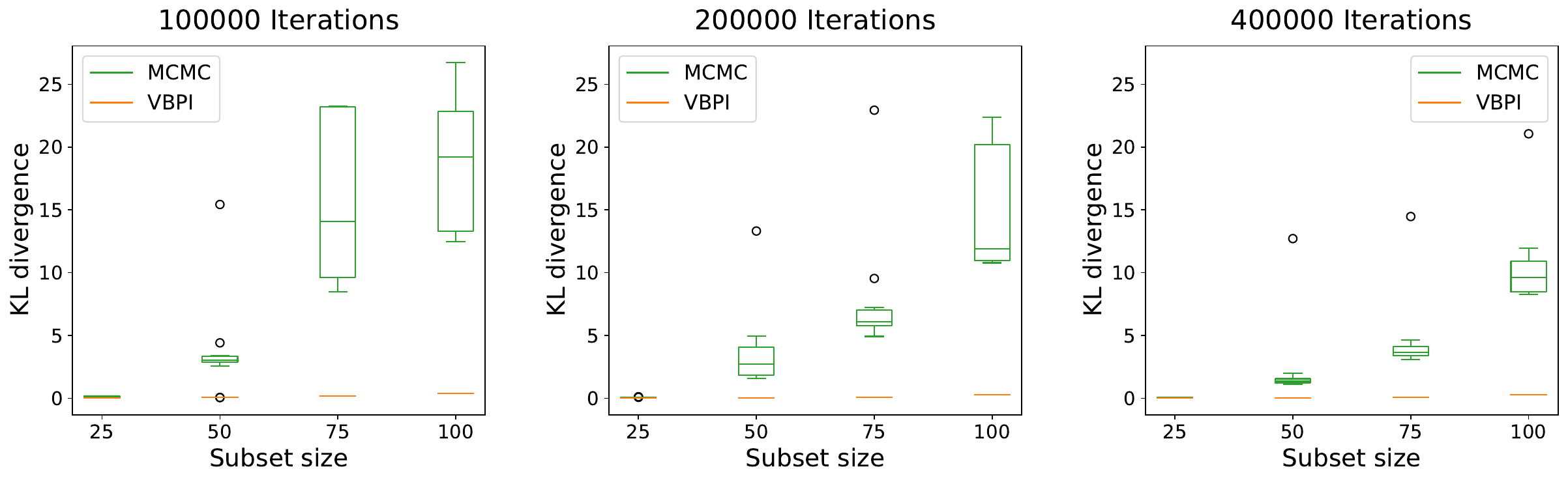}
%\hspace{2.0cm} (a)  \hspace{4.1cm} (b) \hspace{4.3cm} (c)\hspace{1.5cm}
\caption{Scalability of VBPI and MCMC on Influenza Data. ({\bf Left}): KL divergence after $100,000$ iterations. ({\bf Middle}): KL divergence after $200,000$ iterations. ({\bf Right}): KL divergence after $400,000$ iterations. KL results were obtained from 10 independent runs. %The letters in parentheses indicate different subsplit support estimation methods for VBPI.
}\label{fig:flu}
\end{center}
\vspace{-0.2in}
\end{figure}

\begin{table}[t!]
\caption{Marginal likelihood estimation of VBPI and SS across 4 benchmark data sets with nested hemagglutinin sequences.
The marginal likelihood estimates of VBPI and SS were obtained and the results were presented the same way as in Section \ref{sec:lb_mle}.}\label{tab:benchmark}
\begin{center}
\scalebox{1.0}{
%\begin{threeparttable}
\begin{tabular}{lllll}
\toprule
Methods  & Subset size 25 & Subset size 50 & Subset size 75 & Subset size 100\\
\midrule
VBPI  & -13378.38(0.06) & -18615.40(0.16) & -23681.85(0.27) & -28556.96(0.36)  \\
SS & -13378.23(0.24)& -18615.82(1.57) & -23647.14(13.25) & -28176.80(47.16) \\
\bottomrule  % Please only put a hline at the end of the table
\end{tabular}
%\end{threeparttable}
}
\end{center}
\end{table}

%For a benchmarking exercise that incorporated an increasing number of sequences while maintaining a reasonable posterior support, we curated a set of influenza hemagglutinin (HA) sequences.
To benchmark the scalability of VBPI compared to MCMC, we curated a set of influenza data with an increasing number, i.e. 25, 50, 75, 100, of nested hemagglutinin (HA) sequences.
The sequences were obtained from the NIAID Influenza Research Database (IRD)~\citep{Zhang2017-lp} through the web site at \url{https://www.fludb.org/}, downloading all complete HA sequences that passed quality control, which were then subset to H7 sequences, and further downsampled using the Average Distance to the Closest Leaf (ADCL) criterion~\citep{Matsen2013-ue}.
The purpose of this curation was to get sequence subsets with a suitable amount of phylogenetic uncertainty to benchmark the scalability of VBPI against MCMC, which are representative of moderately-dense sampling of viral sequences.
The sequence subsets are available in the \texttt{vbpi-torch} GitHub repository.
As before, we formed our ground truth posteriors from extremely long MCMC runs of 10 billion iterations (sampled each 1,000 iterations with a burn-in rate of $25\%$) using MrBayes.
For VBPI, we used $K=10$ samples in the training objective with PSP branch length parameterization and the VIMCO gradient estimator for training.
%We also tested two different subsplit support estimation methods for VBPI, i.e., maximum likelihood bootstrapping (BS) and short-run MCMC (SM), where we used 10 replicates of 10,000 ultrafast maximum likelihood bootstrap trees and 10 independent MCMC runs of 1,000,000 iterations (sampled each 100 iterations with the first $25\%$ discarded) respectively.
The subsplit supports for VBPI were estimated based on 10 independent short MCMC runs of 1,000,000 iterations (sampled each 100 iterations with the first $25\%$ discarded) respectively.
For a relatively fair comparison to MCMC, as before, we ran MrBayes with $4$ chains and 5 (i.e.\ 2$\cdot$10/4) times the number of iterations when compared to VBPI, and sampled every 100 iterations with the first $25\%$ discarded as burn-in.
Figure \ref{fig:flu} shows the KL divergence of the estimated posteriors from VBPI and MCMC to the ground truth at 100,000, 200,000, and 400,000 VBPI iterations.
We see that VBPI scales much better than MCMC as the number of sequences increases, especially for a small computational budget (100,000 iterations).
The number of iterations that VBPI takes for convergence appears to be quite steady across different sequence subsets, while MCMC requires many more, especially for large number of sequences.
%For VBPI, short-run MCMC outperforms maximum likelihood bootstrapping as the number of sequences increases, showing that the efficiency of subsplit support estimation method matters when the number of sequences is large.
We also tested the performance of VBPI and SS on marginal likelihood estimation and the results are presented in Table \ref{tab:benchmark}.
We see that VBPI and SS provides similar marginal likelihood estimates when the subset size is small.
As the subset size increases, SS tends to provide higher marginal likelihood estimates than VBPI.
However, the variances of these estimates also increases drastically, indicating the decreasing reliability of those marginal likelihood estimates.
In contrast, the variances of VBPI estimates are much smaller and more stable.
Note that SS has been found to overestimate the marginal likelihood when diffuse priors are used \citep{Baele16} due to the notorious difficulty of MCMC sampling from vague priors.
Here we are following standard practice for unrooted Bayesian phylogenetics by using a uniform distribution on tree topologies, which certainly counts as a diffuse prior.
Given that the SBN approximations are close to the tree topology posteriors (KL divergence is small, see Figure \ref{fig:flu}) and the support of Lognormal distributions covers all positive real numbers, the smaller variances of VBPI estimates here would imply more reliable estimates.
Also, we did observe in our experiments that the estimates given by SS are often dominated by some ``large'' values from one of the multiple runs as it uses the \texttt{logmeanexp} function to summarize the results from multiple runs.
That may explain why SS tends to overestimate the marginal likelihood when the variance is large.
As importance sampling scales exponentially in the KL divergence between the proposal and the target \citep{chatterjee18}, this challenge is also present when using variational approximations for importance sampling.
However, it should be less of a problem in our setting as the divergence between variational approximation and the posterior is much smaller than that between the prior and the posterior for SS; we leave a more thorough investigation for future work.

\section{VBPI for Time Measured Phylogenies}
The VBPI framework introduced in Section \ref{sec:vbpi} is mainly designed for unrooted phylogenetic models.
When it comes to population genetic inference, rooted time-measured phylogenetic trees (also known as \emph{time trees}) are adopted to allow estimation of divergence times between sequences and the underlying population dynamics \citep{Pybus2000, Drummond05, Sagulenko18}.
The ``branch lengths'' of a time tree, therefore, correspond directly to elapsed time and each node of the tree is associated with a calendar time that reflects its known or inferred date\footnote{These rescaled branch lengths will be translated back into standard branch lengths that reflect the amount of molecular evolution via molecular clock models. We will revisit it later in Section \ref{sec:vp_rooted}.}.
%For the sake of simplicity, we call this kind of phylogenetic trees \emph{time trees}.
Compared to unrooted trees, time trees are often subject to more complicated constraints induced by the additional sampling times for the data sequences and the evolutionary direction of the tree topologies.
As a result, the priors used for time trees are more complicated than those for unrooted trees, and the variational approximations need to be more carefully designed to handle these constraints efficiently.
In this section, we will describe several essential ingredients that extend the VBPI framework to time tree models.
We start with a brief introduction to the coalescent model that serves as a popular prior for time trees, followed by two typical coalescent priors that allow efficient computation of the coalescent likelihoods.
We then discuss a tree-based reparameterization for constructing variational approximations to the constrained posteriors of time trees, and appropriate variational parameterization techniques for branch length distributions over tree topologies.

\subsection{Coalescent Background}\label{sec:coalescent_likelihood}
Kingman's coalescent model \citep{Kingman1982} is a commonly used prior for time trees, providing a stochastic process that produces genealogies relating the observed molecular sequences.
Given a random population sample of $N$ sequences, the process starts at a sampling time $t=0$ and proceeds backward in time as $t$ increases, merging lineages one pair at a time until the time to the most recent common ancestor (TMRCA) of the sample is reached.
Those merging times are called \emph{coalescent times}.
The coalescent model is parameterized in terms of the \emph{effective population size}, a dynamic model parameter that controls the way that different lineages merge along time.
Let $N_e(t)$ be the time-dependent effective population size.
Suppose we observe a time tree of $N$ sequences sampled at time $0$, and the internal nodes have coalescent times $t_{N-1} < \cdots < t_1$.
The joint density of the coalescent times $t_N=0<t_{N-1} < \cdots < t_1$ is obtained by multiplying the conditional probabilities for each coalescent event \citep{Griffiths1994}
\begin{equation}\label{eq:isochronous_coalescent}
p(\bm{t}|N_e(t)) = \prod_{k=2}^N\frac{A_k}{N_e(t_{k-1})}\exp\left(-\int_{t_k}^{t_{k-1}}\frac{A_k}{N_e(t)}dt\right),
\end{equation}
where $\bm{t}=(t_1,\ldots, t_N)$ and $A_k={{k}\choose{2}}$.
The sampling time may be different as well, see more details on coalescent models with \emph{heterochronous} sampling times in Appendix E.
%In practice, we may have different sampling times when the observed organisms are evolving rapidly.
%In this case, the standard coalescent model can be generalized to account for such \emph{heterochronous} sampling \citep{Rodrigo1999,Felsenstein1999-cm}.
%Consider the ordered node times $t_{2N-1}=0<t_{2N-2}<\cdots< t_1$, which include both the sampling times for the tip nodes and the coalescent times for the internal nodes.
%Let $\ell_k$ denote the number of lineages co-existing in the time interval $(t_{k}, t_{k-1})$ between node $k$ and node $k-1$.
%Let $\mathcal{C}$ denote the set of indices for the internal nodes.
%The density for heterochronous coalescent models can be obtained by modifying density \eqref{eq:isochronous_coalescent} as
%\begin{equation}\label{eq:heterochronous_coalescent}
%p(\bm{t}|N_e(t)) = \prod_{k\in \mathcal{C}}\frac{A_{k+1}}{N_e(t_{k})}\cdot\prod_{k=2}^{2N-1}\exp\left(-\int_{t_k}^{t_{k-1}}\frac{A_k}{N_e(t)}dt\right)
%\end{equation}
%where $\bm{t}=(t_1,\ldots,t_{2N-1}), A_k={{\ell_k}\choose{2}}$.
%Note that density \eqref{eq:isochronous_coalescent} can be viewed as a special case of density \eqref{eq:heterochronous_coalescent} when $t_{k}=0,\forall k\geq N$.
%In what follows, we will refer to density \eqref{eq:heterochronous_coalescent} as the general case.
As the node times $\bm{t}$ are often called node heights, we will use node heights in the sequel.

\subsection{Coalescent Priors}\label{sec:coalescent_prior}
The coalescent densities introduced above require appropriate forms for $N_e(t)$ to be properly estimated.
There are many prior distributions to model the evolution of the population size dynamics over time, known as coalescent priors \citep{Pybus2000, Minin2008, Gill2012}.
In this work, we consider the following two commonly used coalescent priors.

\subsubsection{Constant Population Size}
This model assumes that the population size has remained constant through time, that is,
\[
N_e(t) = N_e=\exp(\gamma_e), \quad t\geq 0
\]
where $\gamma_e$ is the only parameter. In Bayesian settings, we may assume that $\gamma_e$ has a $\mathcal{N}(\mu_0,\sigma_0^2)$ prior.

\subsubsection{Skyride}\label{sec:skyride}
As the true parametric form for $N_e(t)$ is usually unknown, non-parametric coalescent priors are often adopted to provide flexible modeling of complicated demographic trajectories.
Let $t_{k_{N-1}} < t_{k_{N-2}} < \cdots < t_{k_1}=t_1$ denote the coalescent times for the internal nodes.
The Skyride \citep{Minin2008} is a non-parametric model that assumes $N_e(t)$ can change its value only at coalescent times,
\[
N_e(t) = \exp(\gamma_n), \quad t_{k_n} < t \leq t_{k_{n-1}}, \quad n=2, \ldots, N
\]
where $\bm{\gamma}=(\gamma_2,\ldots,\gamma_N)$ are model parameters.
Moreover, to encourage the smoothness of the population size dynamics over time, \citet{Minin2008} introduces a Gaussian Markov random field (GMRF) prior that penalizes the differences between successive components of $\bm{\gamma}$ as
\begin{equation}\label{eq:gmrf}
p(\bm{\gamma}|\lambda) \propto \lambda^{\frac{N-2}{2}} \exp\left(-\frac{\lambda}{2}\sum_{n=2}^{N-1}\frac{(\gamma_{n+1}-\gamma_n)^2}{\delta_n}\right)
\end{equation}
where $\lambda$ is the overall precision of the GMRF and $\{\delta_n\}_{n=2}^{N-1}$ are temporal smoothing coefficients.
%For simplicity, we use a time-unaware GMRF prior that assumes $\delta_k=1, \forall k$.
As no prior knowledge about the smoothness of the effective population size trajectory (i.e.\ function in terms of time) is available, a relatively uninformative gamma prior  is often assigned to the GMRF precision parameter $\lambda$,
\begin{equation}\label{eq:skyride_gamma_prior}
p(\lambda) = \frac{b^a}{\Gamma(a)} \lambda^{a-1}e^{-b\lambda}.
\end{equation}
Note that the gamma prior \eqref{eq:skyride_gamma_prior} is a conjugate prior to the GMRF model \eqref{eq:gmrf}, we therefore can integrate out $\lambda$ to obtain the marginal prior for $\bm{\gamma}$ \citep{BDA},
\begin{equation}
p(\bm{\gamma}) \propto \frac{b^a}{\Gamma(a)}\cdot \frac{\Gamma(a+\frac{N-2}{2})}{\left(b+\frac12\sum_{n=2}^{N-1}\frac{(\gamma_{n+1}-\gamma_n)^2}{\delta_n}\right)^{a+\frac{N-2}2}}.
\end{equation}

These piecewise constant population size trajectories allow easy computation of the coalescent densities introduced in Section \ref{sec:coalescent_likelihood} which we denote as $p(\bm{t}|\bm{\gamma})$ in the sequel.

\subsection{Posterior Approximation and Variational Parameterization}\label{sec:vp_rooted}
As the node heights $\bm{t}$ on a rooted time tree $\tau$ measure the duration of evolution in terms of time, molecular clock models are often employed to transform the node heights to branch lengths that reflect the amount of evolution in terms of the expected number of substitutions per site.
For the sake of simplicity, we use the strict clock model that assumes a constant evolutionary rate $r$ across different branches, leaving more flexible clock models (e.g.~\cite{Drummond2010-xi}) to future work.
Given the evolutionary rate $r$, the branch length from the parent node $i$ to the child node $j$ is $q_{ij} = r(t_i - t_j)$.
We, therefore, denote the branch lengths $\bm{q}$ as a function of $\bm{t}$ and $r$: $\bm{q}=\bm{q}(\bm{t},r)$.
With appropriate choices of the coalescent prior $p(\bm{\gamma})$ as described in Section \ref{sec:coalescent_prior}, and the rate prior $p(r)$, the posterior distribution can be written as
\begin{equation}
p(\tau,\bm{t},\bm{\gamma},r|\bm{Y}) \propto p(\bm{Y}|\tau, \bm{q}(\bm{t},r)) p(\bm{t}|\bm{\gamma}) p(\bm{\gamma})p(r).
\end{equation}
%where $p(\bm{t}|\bm{\gamma})$ is the coalescent likelihood(density) function introduced in Section \ref{sec:coalescent_likelihood}.
Similarly as before, we prescribe a family of variational distributions as
\begin{equation}\label{eq:mean_field_time_tree}
Q_{\bm{\phi},\bm{\psi},\bm{\omega},\nu}(\tau,\bm{t},\bm{\gamma},r) = Q_{\bm{\phi}}(\tau)Q_{\bm{\psi}}(\bm{t}|\tau)Q_{\bm{\omega}}(\bm{\gamma}|\tau)Q_{\nu}(r|\tau)
\end{equation}
where $ Q_{\bm{\phi}}(\tau)$ represents a family of SBN-based distributions over rooted tree topologies, and $Q_{\bm{\psi}}(\bm{t}|\tau),Q_{\bm{\omega}}(\bm{\gamma}|\tau), Q_{\nu}(r|\tau)$ are the conditional distributions for $\bm{t},\bm{\gamma},r$, respectively, that are assumed to be independent given the tree topology $\tau$.

Compared to other parameters, constructing appropriate conditional distributions for the node heights needs more careful consideration.
Unlike the branch lengths for unrooted trees, the node heights $\bm{t}$ for rooted time trees are subject to certain constraints induced by the sampling times and tree topologies.
More specifically, the heights of the parent nodes are required to be greater than the heights of their children and the heights of the tip nodes are given by the corresponding sampling times.
To cope with these constraints, we reparameterize the node heights as follows (similar reparameterizations have also been used before \citep{Yang07, Fourment19,Ji2023-wp}). 
First, we initialize the heights (and the lower bounds) of the tip nodes to the sampling times.
Then, we run a postorder traversal and compute the lower bounds for the node heights when we visit internal node $i$ via
\[
t_i^{\mathrm{lb}} = \max\{t_j^{\mathrm{lb}}: j\in \mathrm{ch}(i)\}.
\]
These lower bounds handle the constraints from the sampling times.
When the sampling times are all $0$, we can simply set these lower bounds to $0$.
Now, we traverse the tree in a preorder fashion and reparametrize the node heights using a non-negative parameter $T$ for the root height and parameters $\theta$ as follows.
We first visit the root node and
\begin{equation}\label{eq:node_height_reparameterization_root}
t_1 = t_1^{\mathrm{lb}} + T, \quad T>0.
\end{equation}
When we visit internal node $i$, we set the node height to (recalling that $\pi_i$ is the parent of $i$)
\begin{equation}\label{eq:node_height_reparameterization_internal}
t_i = t_{\pi_i} - \theta_i (t_{\pi_i}-t_i^{\mathrm{lb}}), \quad 0<\theta_i<1.
\end{equation}
Let $\mathcal{C}$ denote the set of indices for the internal nodes.
The remaining constraints for $T$ and $\bm{\theta}$ can be completely removed with the following transforms
\begin{equation}\label{eq:node_height_reparameterization_alpha}
T = \exp(\alpha_1),\quad \theta_i = \frac{1}{1+\exp(-\alpha_i)}, \;\forall i \in \mathcal{C}\backslash\{1\},
\end{equation}
%This way, we have reparameterized the node heights using $\bm{\alpha}$ that naturally handles the time tree constraints.
where $\{\alpha_i\in \mathbb{R}:i\in\mathcal{C}\}$ are unconstrained parameters for node heights.
This completes our reparameterization of the conditional node height distributions.

Thanks to the hierarchical structure of the tree topology, the probability density function of the transformed node heights is computable given a tractable density for the base distribution of $\bm{\alpha}|\tau$ \citep{Fourment19,Ji2023-wp}.
First, the transform defined in \eqref{eq:node_height_reparameterization_root} and \eqref{eq:node_height_reparameterization_internal} has a triangular Jacobian matrix whose determinant can be readily computed as $\prod_{i\in\mathcal{C}\backslash\{1\}}(t_{\pi_i}-t_i^{\mathrm{lb}})$.
Second, transform \eqref{eq:node_height_reparameterization_alpha} has a diagonal Jacobian matrix whose determinant is $T\cdot\prod_{i\in\mathcal{C}\backslash\{1\}}\theta_i(1-\theta_i)$.
Due to the change of variable formula, we have
\begin{equation}
Q_{\bm{\psi}}(\bm{t}|\tau) = Q_{\bm{\psi}}(\bm{\alpha}|\tau) \left(\prod_{i\in\mathcal{C}\backslash\{1\}}(t_{\pi_i}-t_i^{\mathrm{lb}})\cdot T\cdot\!\!\!\!\prod_{i\in\mathcal{C}\backslash\{1\}}\theta_i(1-\theta_i)\right)^{-1}
\end{equation}
where $Q_{\bm{\psi}}(\bm{\alpha}|\tau)$ is the base distribution.

A simple choice for the base distribution is the diagonal Normal distribution
\[
Q_{\bm{\psi}}(\bm{\alpha}|\tau) = \mathcal{N}\left(\alpha_r|\mu(\tau), \sigma^2(\tau))\right)\cdot \prod_{e\in E(\tau)}\mathcal{N}\left(\alpha_e|\mu(e,\tau),\sigma^2(e,\tau)\right)
\]
where $\alpha_r$ is the parameter associated with the root height parameter $T$.
Similarly as before, we can use the similarity of local splitting patterns to parameterize our node height variational approximations across different rooted tree topologies.
More specifically, we can use the root subsplit for the root height parameter $\alpha_r$.
For the other node height parameters $\{\alpha_e\}_{e\in E(\tau)}$, we can follow the evolutionary direction of the rooted trees and use the clades (i.e., sets of descendant leaf nodes' labels) instead.
Given a rooted tree $\tau$, let $r/\tau$ denote the root subsplit and $e|\tau$ denote the clade for edge $e$.
Let $\mathbb{S}_\mathrm{r},\mathbb{S}_\mathrm{c}$ be the sets of root subsplits and clades in the support respectively.
We assign parameters for $\mu,\sigma$ for each element in $\mathbb{S}_\mathrm{r}\cup\mathbb{S}_\mathrm{c}$ and parameterize the mean and variance of $\bm{\alpha|}\tau$ as follows:
\begin{equation}\label{eq:simple_parameterization_rooted}
\mu(\tau)=\psi_{r/\tau}^\mu, \quad \sigma(\tau) = \psi_{r/\tau}^\sigma,\quad  \mu(e,\tau) = \psi_{e|\tau}^\mu, \quad \sigma(e,\tau)=\psi_{e|\tau}^\sigma.
\end{equation}
Again, more sophisticated parameterizations can be designed by including more tree-dependent local structures.
%similar to the primary subsplit pair parameterization for unrooted trees.
For example, for the root height distribution, we can introduce additional parameters for the associated primary subsplit pairs;
for the other node height distributions, we can introduce additional parameters for the leafward primary subsplit pairs and the subsplits for the clades.
All the associated variational parameters are then summed up to form the mean and variance of the corresponding node height distribution.
This parameterization can be viewed as a straightforward adaption of the primary subsplit parameterization presented in Section \ref{sec:psp} and we refer to it as the primary subsplit parameterization in the sequel.

Compared to the node height parameter $\bm{\alpha}$, the population size parameter $\bm{\gamma}$ and rate parameter $r$ are often much less related to the tree topology $\tau$ when a strict clock model is assumed.
For the sake of simplicity, we remove the dependency of $\bm{\gamma}$ and $r$ on $\tau$ from our variational approximation and assume $Q_{\bm{\omega}}(\bm{\gamma})$ and $Q_\nu(r)$ take diagonal Normal and Log-normal densities respectively.
The best variational approximation from our variational family \eqref{eq:mean_field_time_tree} then can be found by maximizing the multi-sample lower bound
\[
L^K(\bm{\phi},\bm{\psi},\bm{\omega},\nu) = \mathbb{E}_{Q_{\bm{\phi},\bm{\psi},\bm{\omega},\nu}(\tau^{1:K},\bm{t}^{1:K},\bm{\gamma}^{1:K},r^{1:K})}\log\left(\frac1K\sum_{i=1}^K\frac{p(\bm{Y}|\tau^i,\bm{q}^i)p(\bm{t}^i|\bm{\gamma}^i)p(\bm{\gamma}^i)p(r^i)}{Q_{\bm{\phi}}(\tau^i)Q_{\bm{\psi}}(\bm{t}^i|\tau^i)Q_{\bm{\omega}}(\bm{\gamma}^i)Q_\nu(r^i)}\right)
\]
as before, using stochastic gradient estimators introduced in Section \ref{sec:stochastic_gradient_estimators}.

\section{Experiments on Time Trees}\label{sec:experiments_rooted}
%\subsection{Variational Bayesian Inference for Time Trees}
In this section, we analyzed two data sets to evaluate the accuracy of the posterior approximations provided by the proposed VBPI algorithms on time tree models.
As before, we focus on the most challenging part: joint learning of the tree topology, node heights and the population size parameters, and assume the simple \citet{Jukes69} substitution model.
The traditional MCMC analyses were done using BEAST \citep{Suchard18}, a commonly used MCMC implementation for phylogenetic analysis on time trees.
In the first example, we analyzed a set of heterochronous Dengue virus sequences \citep{Lanciotti97} under the strict clock model and a constant size coalescent prior on the tree.
This data set contains 17 whole-genome sequences (each with 1485 nucleotide sites) and can be downloaded from the BEAST Github repository (\url{https://github.com/beast-dev/beast-mcmc/tree/master/examples/Data}).
The corresponding phylogenetic posterior is informative enough to allow us to compare VBPI and the classic MCMC method on the tree topology posterior estimation.
The (rooted) subsplit support was gathered from 10 independent short MCMC runs of 1,000,000 iterations (sampled each 100 iterations with the first $25\%$ discarded).
In the second example, we analyzed a set of 63 RNA sequences of type 4 from the E1 region of the hepatitis C virus (HCV) genome that were sampled in 1993 in Egypt.
We used a fixed substitution rate of $7.9\times 10^{-4}$ per site per year, as suggested in previous studies \citep{Pybus01, Fourment19}.
We considered both the constant size coalescent prior and the Bayesian skyride prior on the tree.
For the Bayesian skyride tree prior, we used a Gaussian Markov random field prior for the effective population size trajectory and a relatively uninformative gamma prior $\Gamma(0.001,0.001)$ for the precision parameter \citep{Minin2008}.
This allows us to marginalize out the precision parameter and hence remove it from the set of parameters that require variational approximations (see Section \ref{sec:skyride}).
We used a uniform GMRF smoothing with equal temporal smoothing coefficients $\delta_2=\cdots = \delta_{62}=1$.
Similarly, the subsplit support was obtained from 10 independent short MCMC runs of 1,000,000 iterations (sampled each 2,000 iterations with the first $50\%$ discarded).
In both examples, the ground truth posteriors were estimated from extremely long MCMC runs of 10 billion iterations (sampled each 10,000 iterations with the first $25\%$ discarded as burn-in).
%The reported KL divergences are over the discrete collection of phylogenetic tree topologies, from trained SBN distributions to the ground truth, and a low KL divergence means a high quality approximation to the exact posterior of the tree topologies.
The KL divergences from trained SBN distributions to the ground truth are reported, similarly as previously done in Section \ref{sec:experiments_unrooted} (see equation \eqref{eq:kl_topology}).
Moreover, we also used similar parameterization strategies, i.e., the simple split-based parameterization and the primary subsplit pair (PSP) parameterization (see Section \ref{sec:vp_rooted}), a similar annealing schedule $\beta_t=\min(1, 0.001+t/50,000)$, and the same training objectives and stochastic gradient estimators as in Section \ref{sec:experiments_unrooted}.

\subsection{Tree Topology Posterior Approximation}\label{sec:tree_posterior_rooted}
\begin{figure}[t!]
\begin{center}
\includegraphics[width=\textwidth]{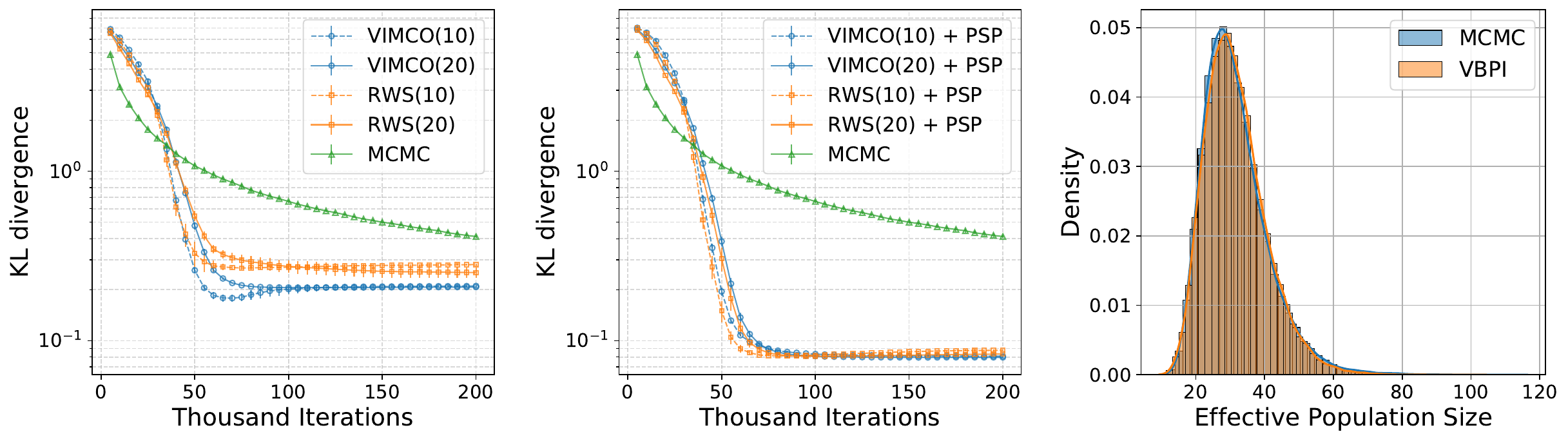}
\caption{Performance on the Dengue virus data set. ({\bf Left}): KL divergence for methods that use the simple split-based parameterization for the branch length distributions. ({\bf Middle}): KL divergence for methods that use PSP. ({\bf Right}): The effective population size posterior estimates.
The number in parentheses specifies the number of samples used in the training objective. The error bar shows one standard deviation over 10 independent runs.}\label{fig:kl_comparison_rooted}
\end{center}
\end{figure}

We first analyzed the heterochronous Dengue virus data to examine the performance of VBPI for tree topology posterior estimation on time tree models.
The left and middle plots in Figure~\ref{fig:kl_comparison_rooted} show the KL divergence to the ground truth as a function of the number of parameter updates.
The results for methods using the simple split-based parameterization are shown in the left plot, and the results for methods using PSP are shown in the middle plot.
Unlike the case of unrooted trees, we see that increasing the number of samples $K$ does not necessarily lead to significantly better tree topology posterior approximations even when the simple split-based parameterization was adopted.
However, a better modeling of between-tree variation of the branch length distributions is still beneficial for all methods with different stochastic gradient estimators and numbers of samples $K$.
%More specifically, we see that using a more flexible PSP parameterization, not only the gaps between different methods become considerably smaller, but also the KL divergences of all methods were significantly improved.
More specifically, we see two outcomes when using a more flexible PSP parameterization: first, the gaps between different methods become considerably smaller, and second, the KL divergences of all methods were significantly improved.
This suggests a moderate $K$ and a flexible cross-topology branch length parameterization would be a favorable combination for VBPI on time trees.
As before, we see that RWS learns faster at the beginning while VIMCO performs better in the end.
To benchmark the learning efficiency of VBPI, we also compared to MCMC using BEAST 1.10.4 \citep{Suchard18}.
For that, we conducted 10 independent MCMC runs in BEAST, each for 8,000,000 iterations, and sample every 100 iterations.
For each run, we computed the KL divergence to the ground truth every 200,000 iterations with the first $25\%$ discarded as burn-in.
As before, we compare 40 times the number of MCMC iterations with the number of 20-sample objective VBPI iterations to account for the additional gradient computation and multiple samples required by the training objectives.
Similarly, we see that after a relatively slower starting phase, VBPI methods outperform MCMC and arrive at good approximations (especially those with PSP) with much fewer iterations.

\subsection{Lower Bound and Population Dynamics Estimation}
\begin{figure}[t!]
\begin{center}
\includegraphics[width=\textwidth]{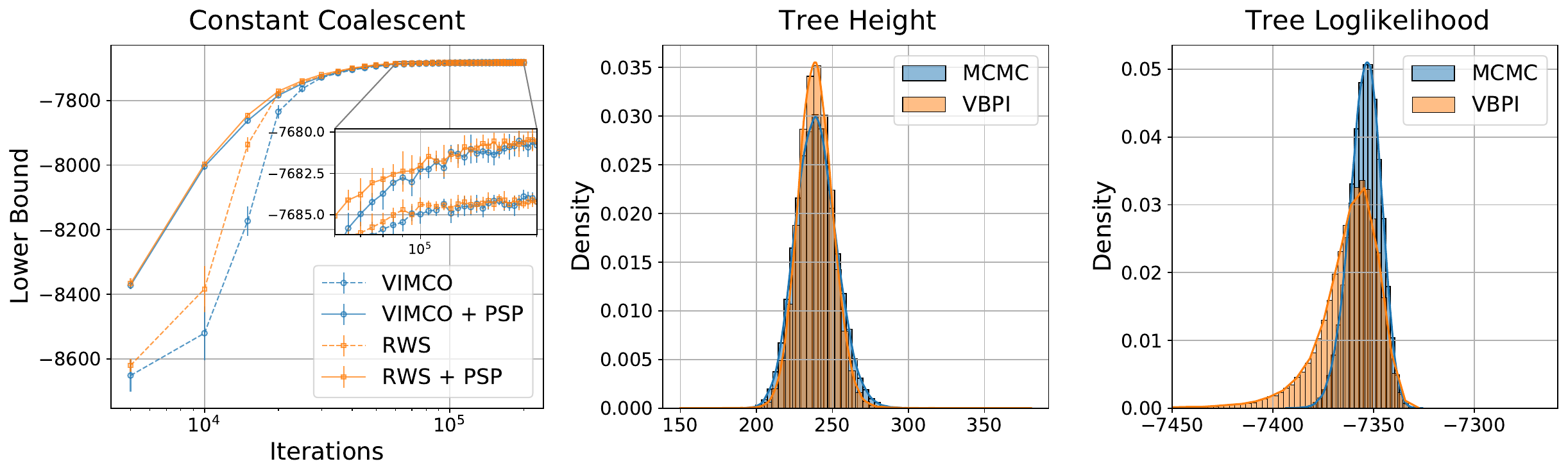}
\caption{Performance on the HCV data set with a constant coalescent prior. ({\bf Left}): Lower bounds. ({\bf Middle}): Tree height posteriors. ({\bf Right}): Tree Likelihoods from posterior samples.
All VBPI methods use 10 samples in the training objective. The error bar shows one standard deviation over 10 independent runs.}\label{fig:lower_bound_constant}
\end{center}
\end{figure}

\begin{figure}[t!]
\begin{center}
\includegraphics[width=\textwidth]{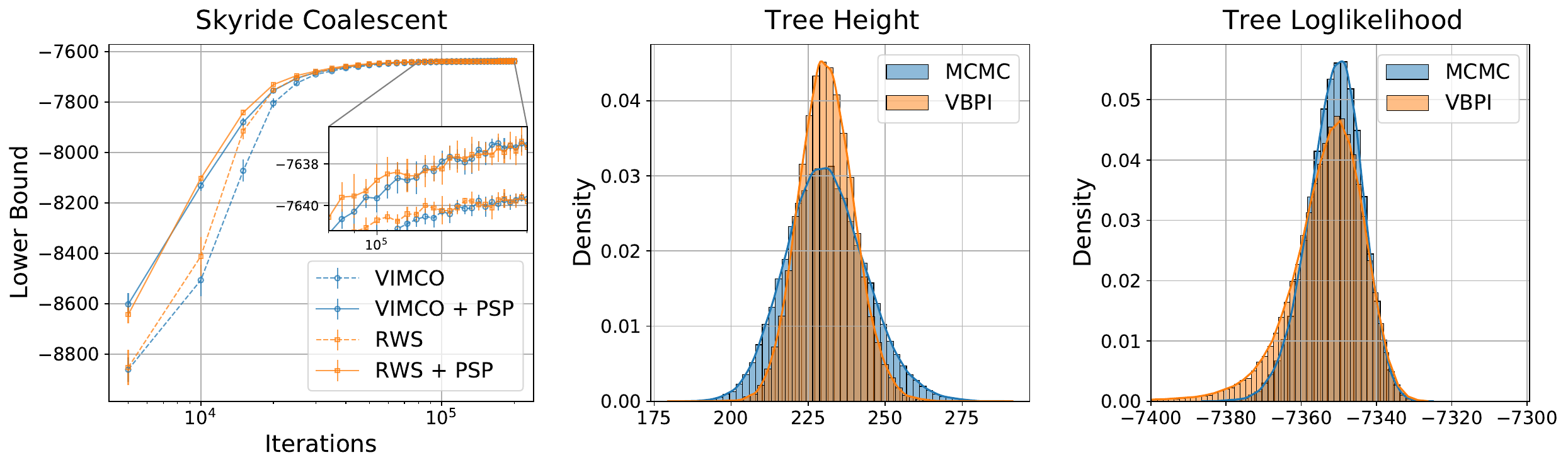}
\caption{Performance on the HCV data set with the skyride coalescent prior. ({\bf Left}): Lower bounds. ({\bf Middle}): Tree height posteriors. ({\bf Right}): Tree Likelihoods from posterior samples.
All VBPI methods use 10 samples in the training objective. The error bar shows one standard deviation over 10 independent runs.}\label{fig:lower_bound_skyride}
\end{center}
\vspace{-0.2in}
\end{figure}

In this section, we analyzed the convergence behavior of VBPI and its approximation accuracy on parameters other than the tree topology on a set of HCV sequences using both a constant coalescent prior and the skyride model with a fixed substitution rate.
The skyride model is flexible enough to capture the interesting underlying population dynamics for this data set \citep{Minin2008}.
As suggested in Section \ref{sec:tree_posterior_rooted}, we used $K=10$ samples in the training objective in all experiments.
The posterior approximations from VBPI were obtained from 75,000 samples\footnote{The same number of MCMC samples were used to form the ground truth posterior estimate.} after training using the VIMCO gradient estimator.
Figure~\ref{fig:lower_bound_constant} and \ref{fig:lower_bound_skyride} show the results for the constant coalescent prior and the skyride model respectively.
In both figures, the left plots show the evidence lower bounds (ELBO) as a function of the number of parameter updates.
We see that VIMCO and RWS gradient estimators perform similarly to each other with RWS being a little bit faster at the beginning.
For both VIMCO and RWS, using the PSP parameterization provides much tighter lower bounds and smoother training curves, which again, demonstrates the importance of the flexibility of branch length parameterization across different topologies.
The middle plots compare the posterior estimates of the tree height from MCMC and VBPI.
We see that VBPI is more conservative than MCMC on tree height posterior estimation.
This is partly due to the parameterization of the node heights in VBPI (see Section \ref{sec:vp_rooted}) where the tree height (i.e., root node height) has a strong impact on all the rest of the node heights on the tree.
The right plots show the tree likelihoods of samples from MCMC and VBPI.
We see that the likelihood values of trees sampled from VBPI are similar to those from MCMC, although more diffuse.

\begin{figure}[t!]
\begin{center}
\includegraphics[width=\textwidth]{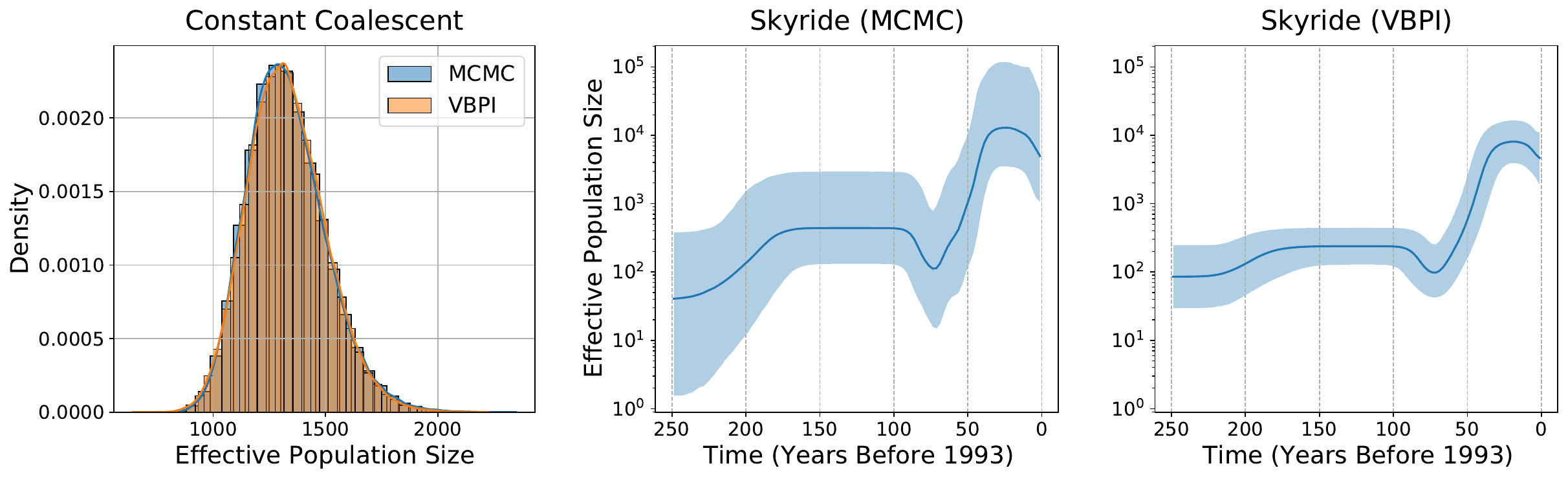}
\vspace{-0.1in}
\caption{Posterior approximation of the effective population size parameters on the HCV data set. ({\bf Left}): Constant coalescent. ({\bf Middle}): Skyride coalescent using MCMC. ({\bf Right}): Skyride coalescent using VBPI. The solid lines represent the posterior medians and the shaded areas represent the $95\%$ credible intervals.}\label{fig:pop_size_hcv}
\end{center}
\vspace{-0.2in}
\end{figure}

Finally, we evaluate the ability of VBPI to approximate the population dynamics of these HCV sequences (Figure~\ref{fig:pop_size_hcv}).
We see that for the simple constant coalescent model (the left plot), VBPI can provide high quality posterior approximation that matches the ground truth well.
The right plot in Figure~\ref{fig:kl_comparison_rooted} shows a similar result on the Dengue virus data set.
For the more challenging skyride model (the middle and the right plot), although the population trajectory estimated from VBPI still matches the ground truth, the $95\%$ credible intervals tend to be narrower than those estimated from MCMC.
This can, at least in part, be explained by the zero-forcing nature of the KL divergence used in VBPI and lack of flexibility of the mean field approximation for the population size parameters (see Section \ref{sec:vp_rooted}).
Overall, VBPI can provide reasonably good posterior approximation of the effective population size, especially when the coalescent model is simple.
We expect this problem of variance underestimation on the effective population size for complicated coalescent models to be alleviated with more flexible variational approximations and leave it for future work.

\section{Conclusion}\label{sec:conclusion}
In this work we propose VBPI as a general variational framework for Bayesian phylogenetic inference on the joint parameter space of phylogenetic models.
At the core of VBPI lie subsplit Bayesian networks, an expressive probabilistic graphical model that provides distributions over the tree topology space, and an efficient structured amortization of the branch lengths over different tree topologies.
We show that this provides a suitable variational family of distributions which can be trained by maximizing the multi-sample ELBO using stochastic gradient ascent.
With guided exploration in the tree topology space (enabled by SBNs) and joint learning of the branch length distribution across tree topologies (via amortization), VBPI can provide competitive performance to MCMC methods especially when the posterior is diffuse.
We used different branch length parameterizations and different numbers of samples in the multi-sample EBLO for training, and found that the flexibility of branch length parameterization across tree topologies has a stronger impact on the approximation accuracy and the complexity in optimization.
Moreover, variational approximations provided by VBPI can be readily used for downstream tasks such as marginal likelihood estimation for model comparison via importance sampling.
We report promising numerical results demonstrating the effectiveness and efficiency of VBPI on a benchmark of real data Bayesian phylogenetic inference problems.

The current VBPI framework also has many limitations.
The main limitation is that it requires subsplit support estimation for variational parameterization.
Subsplit support estimation is currently done with a pre-specified number of candidate tree topologies.
As the exact posterior is often unknown, we may need a stopping criterion for sufficient support exploration in practice.
The KL divergence, or equivalently, the ELBO, that are commonly used for variational inference is not suitable in this case due to its mode seeking behavior.
As an alternative, we may conduct multiple short MCMC/bootstrap runs, and check if the newly collected subsplit supports contain significantly novel split patterns that haven't been seen before.
In other lines of research, we are currently working on systematic means of support estimation, generalizing~\cite{Whidden19} to the subsplit support context with several novel objective functions.

In addition, when the data are relatively insufficient and posteriors are diffuse, subsplit support estimation becomes challenging.
However, compared to classical MCMC approaches in phylogenetics that need to traverse the enormous support of posteriors on complete trees to accurately evaluate the posterior probabilities, the subsplit-based parameterization in VBPI has a natural advantage in that it alleviates this issue by factorizing the uncertainty of complete tree topologies into local structures.
This way, the SBN approximation given by VBPI can be viewed as a compact representation of the topological uncertainty of phylogenetic trees that can be learned in a timely manner.
On the other hand, when there is minimal signal for structure of a region of the tree, a variational model will have a large number of parameters in that region.

Another limitation is that joint inference on model parameters could be challenging, especially when there exists an unidentifiability (e.g., rate matrix and branch length).
Although the resulting posterior would still be proper when an appropriate prior is used for the related parameters, the correlation among these related parameters would be more complex.
When this happens (e.g., relaxed clock models for time trees), our current variational approximation construction may be insufficient and more sophisticated design may be needed.
For example, we may expand all individual variational building blocks with as many dimensions as global parameters.
To reduce the number of parameters, we can attribute the global parameters to different variational building blocks when there exists conditional independence structures in the model.
We can also try tying together parameters over different variational building blocks using neural networks.
Another option is to provide more flexible variational families for the joint distribution of related parameters (e.g., normalizing flows which we discuss later for future work), instead of a product of conditionals.

Another potential limitation concerns optimization.
Since the ELBO objective is often multimodal, the training of VBPI may end up with some local mode and hence lead to suboptimal variational approximations.

There are many extensions for future work.
The first extension is to construct more flexible variational family of distributions.
The current branch length variational distributions in VBPI are relatively simple and can be improved by using more flexible architectures (e.g., normalizing flows \citep{NF}) to reduce the approximation error and amortization error.
Recent work \citep{VBPI-NF, VBPI-GNN} shows promising results in this direction where properly incorporating the inductive biases of phylogenetic models (e.g., the edges are orderless across tree topologies) seems to be the key.
Similarly, we can design more flexible variational approximations for the effective population size parameters to alleviate the variance underestimation problem for time tree models.

Another potential extension is to find more efficient ways to infer the support of subsplits.
Besides bootstrapping and short MCMC runs, good candidate trees can also be obtained from fast heuristic search methods \citep{Whidden19}.
Instead of fixing the support, we can also grow a dynamic support of candidate subsplits and keep adding ``good'' subsplits and removing ``bad'' subsplits in a manner analogous to maximum-likelihood phylogenetic inference.
Moreover, we can also design appropriate stopping criteria as discussed above for these methods.
Besides SBNs, one can also design distributions that have support over the entire tree topology space, eliminating the necessity for support estimation \citep{xie2023artree, mimori2023geophy, zhou2024phylogfn}. 

When the variational approximations are not close enough to the exact posterior, pure importance sampling based marginal likelihood estimation using VBPI may become prohibitive.
In this case, we can instead use these variational approximations as reference distributions for GSS approaches \citep{GSS, Holder14, Baele16}.

For the sake of simplicity, we only considered \citet{Jukes69} substitution model and the strict clock model in this paper.
Future methods will also consider including other interesting components of phylogenetic models, such as more complicated substitution models and the relaxed clock model, into our variational formulation.

% Acknowledgements should go at the end, before appendices and references

\acks{This work supported by NIH grant AI162611.
The research of Cheng Zhang was supported in part by National Natural Science Foundation of China (grant no. 12201014 and grant no. 12292983), the National Engineering Laboratory for Big Data Analysis and Applications, the Key Laboratory of Mathematics and Its Applications (LMAM) and the Key Laboratory of Mathematical Economics and Quantitative Finance (LMEQF) of Peking University.
Frederick Matsen is an investigator of the Howard Hughes Medical Institute.
The authors would like to thank Mathieu Fourment, Michael Karcher, Marc Suchard, and Christiaan Swanepoel for discussions.
The authors are also grateful for the invaluable feedback from our anonymous reviewers and the Action Editor.}


\newpage

\appendix

\section*{Appendix A. The PSP Parameterization}
\begin{figure}[h!]
\begin{center}
\input{figs/subsplit_pair}
\caption{%An illustration of the PSP branch length parameterization.
%({\bf Left}): Model architectures for general SBNs.
%The solid full and complete binary tree network is $\mathcal{B}_{\mathcal{X}}^\ast$.
%The dashed arrows represent the additional dependence structures.
%({\bf Right}):
Branch length parameterization using primary subsplit pairs, which is the sum of parameters for a split and its neighboring subsplit pairs.
Edge $e$ represents a split $(W,Z)$. Parameterization for the variance is the same as for the mean.}\label{fig:subsplitpair}
\end{center}
\end{figure}

\section*{Appendix B. Proofs of core properties of SBNs}

% Note: in this sample, the section number is hard-coded in. Following
% proper LaTeX conventions, it should properly be coded as a reference:

%In this appendix we prove the following theorem from
%Section~\ref{sec:textree-generalization}:

In this section, we prove some important properties of subsplit Bayesian networks given in Section \ref{sec:sbn}.

\vspace{1em}
\noindent
{\bf Lemma \ref{lemma:sbnrep}}
{There is a one-to-one mapping between rooted tree topologies on $\mathcal{X}$ and compatible node assignments.}
\vspace{1em}
\begin{proof}
%It suffices to prove the lemma for the minimum SBN $\mathcal{B}_\mathcal{X}^\ast$.
For any rooted tree $\tau$, with natural indexing, a compatible node assignment is easy to obtain by following the splitting process and obeying the compatibility requirement introduced in Section \ref{sec:sbn}. Moreover, the depth of the splitting process is at most $N-1$, where $N$ is the size of $\mathcal{X}$. To show this, consider the maximum size of clades, $d_i$, in the $i$-th level of the process. The first level is for the root split, and $d_1 \le N-1$ since the subclades need to be proper subset of $\mathcal{X}$. For each $i>1$, if $d_i>1$, the splitting process continues and the maximum size of clades in the next level is at most $d_i-1$, that is $d_{i+1}\leq d_i-1$. Note that the split process ends when $d_i=1$, its depth therefore is at most $N-1$. The whole process is deterministic, so each rooted tree can be uniquely represented as a compatible assignment of $\mathcal{B}_\mathcal{X}^\ast$. On the other hand, given a compatible node assignment, we can simply remove all singleton clades. The remaining net corresponds to the true splitting process of a rooted tree, which then can be simply restored.
\end{proof}

\noindent
{\bf Proposition \ref{prop:consistency}}
{For any consistently parameterized SBN, $\sum_\tau p_{\mathrm{sbn}}(T=\tau) = 1$.}
\vspace{1em}
\begin{proof}
Denote the set of all SBN assignments as $\mathcal{A}$. For any assignment $a = \{S_i=s_i^a\}_{i\ge 1}$, we have
\begin{equation}
p_{\mathrm{sbn}}(a) = p(S_1=s_1^a) \prod_{i>1}p(S_i=s_i^a|S_{\pi_i} = s_{\pi}^a) > 0 \Rightarrow \text{ $a$ is compatible}.
\end{equation}
As a Bayesian network,
\[
\sum_{a\in \mathcal{A}}p_{\mathrm{sbn}}(a) = \sum_{s_1}p(S_1=s_1)\prod_{i>1}\sum_{s_i}p(S_i=s_i|S_{\pi_i}=s_{\pi_i}) = 1.
\]
By Lemma \ref{lemma:sbnrep},
\[
\sum_{T}p_{\mathrm{sbn}}(T) = \sum_{a \sim {\mathrm{compatible}}} p_{\mathrm{sbn}}(a) = \sum_{a\in \mathcal{A}}p_{\mathrm{sbn}}(a) = 1.
\]
\end{proof}
Here $a \sim {\mathrm{compatible}}$ denotes a sum over compatible node assignments as per Definition~\ref{def:compatible}.

\section*{Appendix C. Gradients for The Multi-sample Objectives}
%\label{sec:gradSI}

In this section we will derive the gradients for the multi-sample objectives introduced in Section \ref{sec:vbpi}.
The derivation presented here is a straightforward adaption of the stochastic gradient estimators introduced before \citep{VIMCO, VAE, RWS}, which we write out make explicit the slightly unusual setting of phylogenetics, which includes both a discrete tree topology and associated continuous parameters.
We start with the lower bound

\begin{align*}
L^K({\bm{\phi}}, {\bm{\psi}}) &= \mathbb{E}_{Q_{{\bm{\phi}},{\bm{\psi}}}(\tau^{1:K},\; {\bm{q}}^{1:K})}\log\left(\frac1K\sum_{j=1}^K\frac{p({\bm{Y}}|\tau^j, {\bm{q}}^j) p(\tau^j, {\bm{q}}^j)}{Q_{{\bm{\phi}}}(\tau^j)Q_{\bm{\psi}}({\bm{q}}^j|\tau^j)}\right)\\
&= \mathbb{E}_{Q_{{\bm{\phi}},{\bm{\psi}}}(\tau^{1:K},\; {\bm{q}}^{1:K})}\log\left(\frac1K\sum_{j=1}^K f_{{\bm{\phi}}, {\bm{\psi}}}(\tau^j, {\bm{q}}^j)\right).
\end{align*}
Using the product rule and noting that $\nabla_{\bm{\phi}} \log f_{{\bm{\phi}},{\bm{\psi}}}(\tau^j,{\bm{q}}^j) = - \nabla_{\bm{\phi}} \log Q_{\bm{\phi}}(\tau^j)$,
\begin{align*}
\nabla_{{\bm{\phi}}}L^k({\bm{\phi}},{\bm{\psi}}) &= \mathbb{E}_{Q_{{\bm{\phi}},{\bm{\psi}}}(\tau^{1:K},\; {\bm{q}}^{1:K})} \nabla_{\bm{\phi}} \log\left(\frac1K\sum_{j=1}^K f_{{\bm{\phi}}, {\bm{\psi}}}(\tau^j, {\bm{q}}^j)\right) + \\
&\hphantom{===} \mathbb{E}_{Q_{{\bm{\phi}},{\bm{\psi}}}(\tau^{1:K},\; {\bm{q}}^{1:K})} \sum_{j=1}^K \frac{\nabla_{\bm{\phi}} Q_{\bm{\phi}}(\tau^j)}{Q_{\bm{\phi}}(\tau^j)} \log\left(\frac1K\sum_{i=1}^K f_{{\bm{\phi}}, {\bm{\psi}}}(\tau^i, {\bm{q}}^i)\right) \\
% Erick note to self: this extra f on top of the fraction comes because we are replacing a grad f with a grad log f.
& = \mathbb{E}_{Q_{{\bm{\phi}},{\bm{\psi}}}(\tau^{1:K},\; {\bm{q}}^{1:K})} \sum_{j=1}^K \frac{f_{{\bm{\phi}},{\bm{\psi}}}(\tau^j,{\bm{q}}^j)}{\sum_{i=1}^Kf_{{\bm{\phi}},{\bm{\psi}}}(\tau^i,{\bm{q}}^i)} \nabla_{\bm{\phi}} \log f_{{\bm{\phi}},{\bm{\psi}}}(\tau^j,{\bm{q}}^j) + \\
&\hphantom{===}  \mathbb{E}_{Q_{{\bm{\phi}},{\bm{\psi}}}(\tau^{1:K},\; {\bm{q}}^{1:K})} \sum_{j=1}^K  \log\left(\frac1K\sum_{i=1}^K f_{{\bm{\phi}}, {\bm{\psi}}}(\tau^i, {\bm{q}}^i)\right) \nabla_{\bm{\phi}} \log Q_{{\bm{\phi}}}(\tau^j) \\
& = \mathbb{E}_{Q_{{\bm{\phi}},{\bm{\psi}}}(\tau^{1:K},\; {\bm{q}}^{1:K})} \sum_{j=1}^K \left(\hat{L}^K({\bm{\phi}},{\bm{\psi}}) - \tilde{w}^j\right) \nabla_{\bm{\phi}} \log Q_{\bm{\phi}}(\tau^j).
\end{align*}
This gives the naive gradient of the lower bound w.r.t. ${\bm{\phi}}$.

Using the reparameterization trick, the lower bound has the form
\begin{align*}
L^K({\bm{\phi}},{\bm{\psi}}) &= \mathbb{E}_{Q_{{\bm{\phi}},{\bm{\epsilon}}}(\tau^{1:K},{\bm{\epsilon}}^{1:K})}\log\left(\frac1K\sum_{j=1}^K\frac{p({\bm{Y}}|\tau^j,g_{\bm{\psi}}({\bm{\epsilon}}^j|\tau^j))p(\tau^j, g_{\bm{\psi}}({\bm{\epsilon}}^j|\tau^j))}{Q_{\bm{\phi}}(\tau^j)Q_{\bm{\psi}}(g_{\bm{\psi}}({\bm{\epsilon}}^j|\tau^j)|\tau^j)}\right) \\
&= \mathbb{E}_{Q_{{\bm{\phi}},{\bm{\epsilon}}}(\tau^{1:K},{\bm{\epsilon}}^{1:K})}\log\left(\frac1K\sum_{j=1}^K f_{{\bm{\phi}},{\bm{\psi}}}(\tau^j, g_{\bm{\psi}}({\bm{\epsilon}}^j|\tau^j)) \right)
\end{align*}
Since ${\bm{\psi}}$ is not involved in the distribution with respect to which we take expectation,
\begin{align*}
\nabla_{\bm{\psi}} L^K({\bm{\phi}},{\bm{\psi}}) &= \mathbb{E}_{Q_{{\bm{\phi}},{\bm{\epsilon}}}(\tau^{1:K},{\bm{\epsilon}}^{1:K})} \nabla_{\bm{\psi}} \log\left(\frac1K\sum_{j=1}^K f_{{\bm{\phi}},{\bm{\psi}}}(\tau^j, g_{\bm{\psi}}({\bm{\epsilon}}^j|\tau^j)) \right) \\
&= \mathbb{E}_{Q_{{\bm{\phi}},{\bm{\epsilon}}}(\tau^{1:K},{\bm{\epsilon}}^{1:K})} \sum_{j=1}^K \frac{f_{{\bm{\phi}},{\bm{\psi}}}(\tau^j, g_{\bm{\psi}}({\bm{\epsilon}}^j|\tau^j))}{\sum_{i=1}^Kf_{{\bm{\phi}},{\bm{\psi}}}(\tau^i, g_{\bm{\psi}}({\bm{\epsilon}}^i|\tau^i))}\nabla_{\bm{\psi}} \log f_{{\bm{\phi}},{\bm{\psi}}}(\tau^j,g_{\bm{\psi}}({\bm{\epsilon}}^j|\tau^j)) \\
&= \mathbb{E}_{Q_{{\bm{\phi}},{\bm{\epsilon}}}(\tau^{1:K},{\bm{\epsilon}}^{1:K})} \sum_{j=1}^K \tilde{w}^j\nabla_{\bm{\psi}} \log f_{{\bm{\phi}},{\bm{\psi}}}(\tau^j,g_{\bm{\psi}}({\bm{\epsilon}}^j|\tau^j)).
\end{align*}

Next, we derive the gradient of the multi-sample likelihood objective used in RWS
\[
\tilde{L}({\bm{\phi}},{\bm{\psi}}) = \mathbb{E}_{p(\tau,{\bm{q}}|{\bm{Y}})}\log Q_{{\bm{\phi}},{\bm{\psi}}}(\tau, {\bm{q}}).
\]
Again, $p(\tau, {\bm{q}}|{\bm{Y}})$ is independent of ${\bm{\phi}},{\bm{\psi}}$, and we have
\begin{align*}
\nabla_{\bm{\phi}} \tilde{L}({\bm{\phi}},{\bm{\psi}}) &= \mathbb{E}_{p(\tau,{\bm{q}}|{\bm{Y}})} \nabla_{\bm{\phi}} \log Q_{{\bm{\phi}},{\bm{\psi}}}(\tau, {\bm{q}}) \\
&= \mathbb{E}_{Q_{{\bm{\phi}},{\bm{\psi}}}(\tau, {\bm{q}})} \frac{p(\tau, {\bm{q}}|{\bm{Y}})}{Q_{\bm{\phi}}(\tau)Q_{\bm{\psi}}({\bm{q}}|\tau)} \nabla_{\bm{\phi}} \log Q_{{\bm{\phi}},{\bm{\psi}}}(\tau, {\bm{q}}) \\
&= \frac{1}{p({\bm{Y}})} \mathbb{E}_{Q_{{\bm{\phi}},{\bm{\psi}}}(\tau,{\bm{q}})} f_{{\bm{\phi}},{\bm{\psi}}}(\tau, {\bm{q}}) \nabla_{\bm{\phi}} \log Q_{\bm{\phi}}(\tau)\\
& \simeq \sum_{j=1}^K \frac{f_{{\bm{\phi}},{\bm{\psi}}}(\tau^j,{\bm{q}}^j)}{\sum_{i=1}^kf_{{\bm{\phi}},{\bm{\psi}}}(\tau^i,{\bm{q}}^i)} \nabla_{\bm{\phi}} \log Q_{\bm{\phi}}(\tau^j)  \;\text{ with }\; \tau^j, {\bm{q}}^j \mathrel{\overset{\makebox[0pt]{\mbox{\normalfont\tiny\sffamily iid}}}{\sim}} Q_{{\bm{\phi}},{\bm{\psi}}}(\tau,\;{\bm{q}}). \\
&= \sum_{j=1}^K \tilde{w}^j \nabla_{\bm{\phi}} \log Q_{\bm{\phi}}(\tau^j)
\end{align*}
The second to last step uses self-normalized importance sampling with $K$ samples. $\nabla_{\bm{\psi}}\tilde{L}({\bm{\phi}},{\bm{\psi}})$ can be computed in a similar way.

\begin{figure}[t!]
\begin{center}
\includegraphics[width=\textwidth]{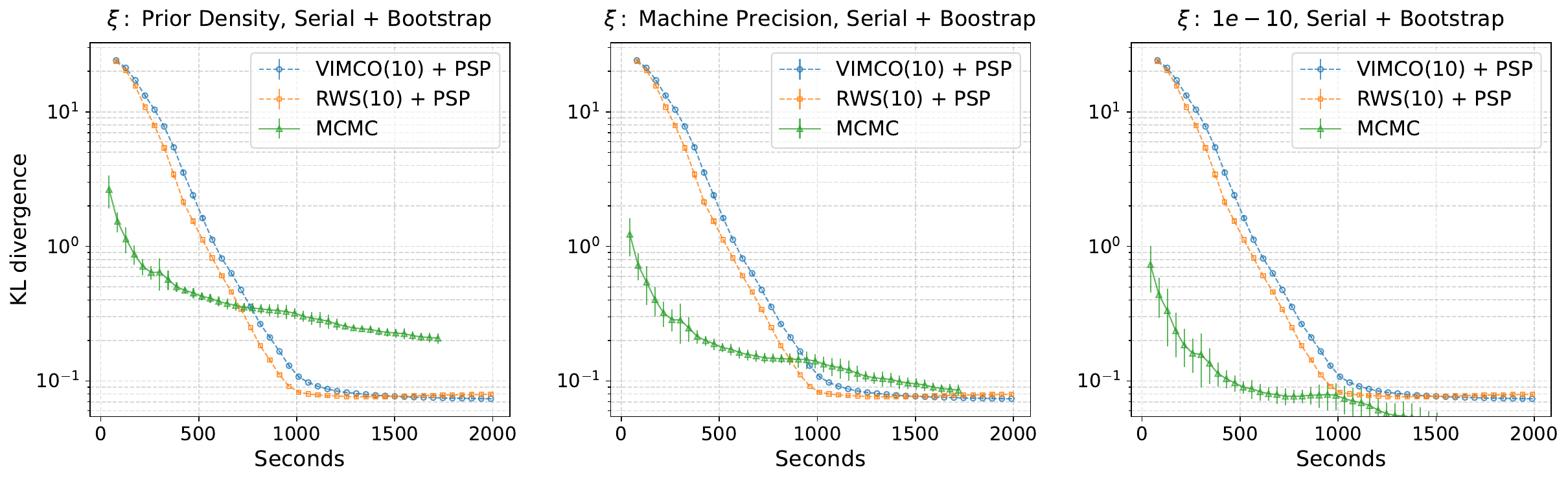}\\
\includegraphics[width=\textwidth]{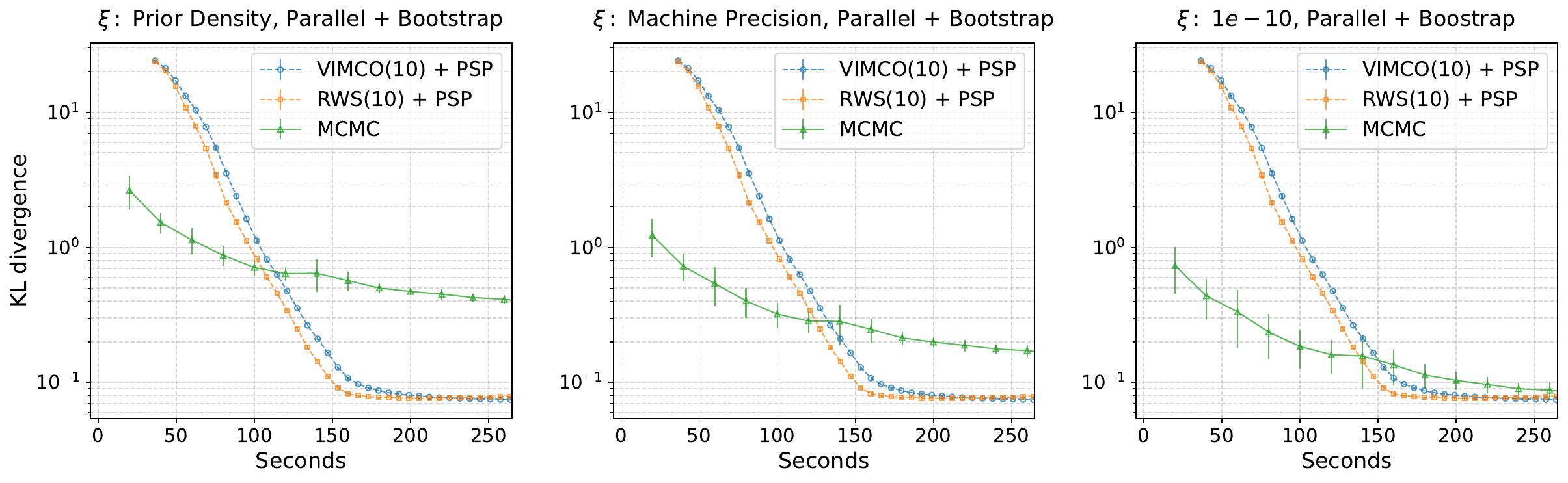}\\
\includegraphics[width=\textwidth]{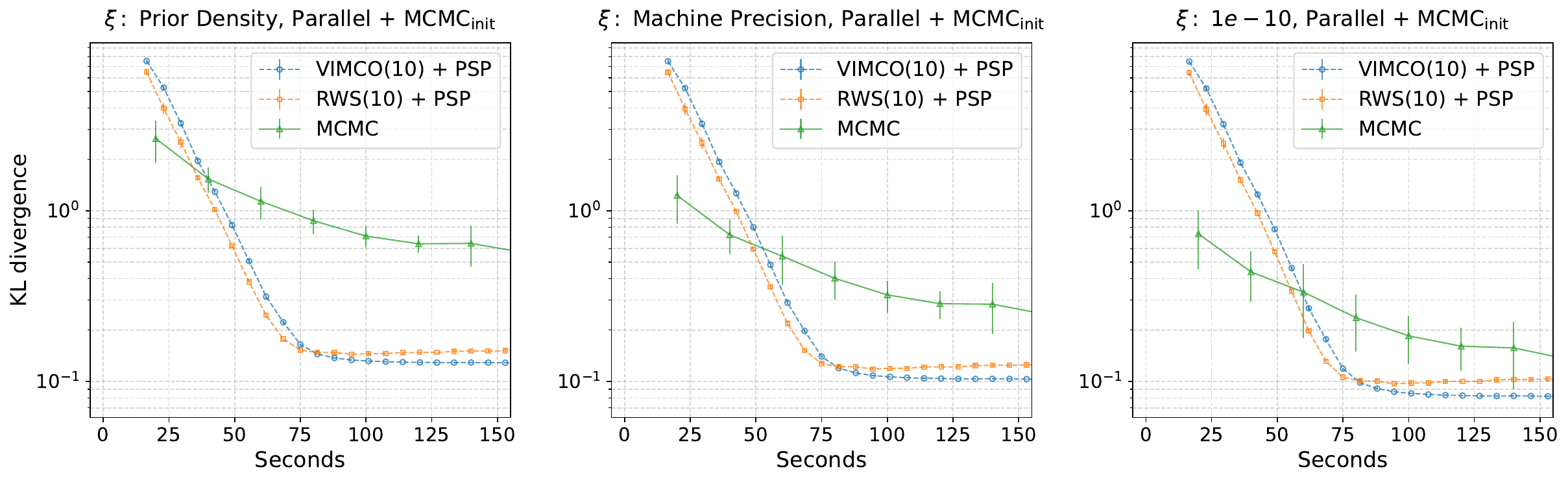}
\caption{A runtime comparison between VBPI and MCMC on DS1. ({\bf Top}): Serial implementation with bootstrap subsplit support. ({\bf Middle}): Parallel implementation with bootstrap subsplit support. ({\bf Bottom}): Parallel implementation with short run MCMC subsplit support. Each column corresponds to a different value of $\xi$. Error bars show one standard deviation over 10 independent runs.}\label{fig:ds1_run_time_comparison}
\end{center}
\vspace{-0.2in}
\end{figure}

\section*{Appendix D. Runtime Comparison between VBPI and MCMC}
We provide approximate wall time comparison results in this section.
That is, we work to understand what the wall time of phylogenetic inference using an optimized implementation of VBPI would be, including the cost of subsplit support estimation.
Our current implementation is in Python, which uses automatic differentiation for likelihood gradient computation; this is inefficient compared to custom gradient code~\citep{Fourment2022-go}.
For example, on the DS1 example below, our python implementation takes about 14200 seconds for 200,000 iterations while the more customized \texttt{bito} library takes only around 1960 seconds for all involved phylogenetic likelihood and gradient computation.
It is not straightforward to compare the number of phylogenetic likelihood computations between VBPI and MCMC as most phylogenetic MCMC moves can leverage the caching trick for more efficient likelihood computation, without leaving a record of how many likelihood operations were performed.

Therefore we use the following strategy to estimate wall clock time, based on the assumption that the phylogenetic gradient will dominate run time.
For VBPI, we record the number of gradient operations needed for a given number of iterations.
For each such iteration, we lower bound the runtime by the likelihood and gradient computation for that sequence alignment using \texttt{bito} (\url{https://github.com/phylovi/bito}), which is a Python-interface C++ library that uses BEAGLE \citep{BEAGLE3} as the backend likelihood and gradient engine.
We used $K=10$ samples in the training objective with PSP branch length parameterization and both the VIMCO and RWS gradient estimators for training.
The runtimes for VBPI also include the cost of computing a sample of tree topologies and gathering the set of subsplit supports from those trees.

For MCMC, we use MrBayes with BEAGLE backend as well.
For each run, we ran MrBayes with 4 chains for 8,000,000 iterations, sampling every 100 iterations, and computed the KL divergence to the ground truth every 200,000 iterations with the first $25\%$ discarded as burn-in.
We examine the performance of VBPI and MCMC on two data sets: DS1 and the influenza data set with 100 sequences introduced in Section \ref{sec:curated_data_benchmark}.

We also investigate the impact that $\xi$ has in the KL computation.
This parameter is the assigned probability for a tree in the ground truth/long MCMC run but not in the approximate distribution (short MCMC or subsplit support).
We test two alternative choices for $\xi$ in the KL computation: the prior density and $1e\text{-}10$.
The prior density corresponds to treating all unsampled tree topologies equally and penalizes MCMC more for not detecting trees in the ground truth samples, especially when the number of sequences is large, while $1e\text{-}10$ is closer to the smallest ground truth posterior estimates for both data sets and hence would bias MCMC towards the ground truth.
Besides serial implementations, we also compared parallel implementations of both method where VBPI can be more advantageous due to the use of multiple-sample training objectives.
We used 4 processors for MPI MrBayes runs as there are 4 chains.
For VBPI, we lower bound the parallel runtimes using \texttt{bito} with 10 threads as $K=10$.
The results were obtained from 10 independent runs for both methods.
All experiments were done on a 2019 Mac Pro (3.2 GHz 16-Core Intel Xeon W).

Figure \ref{fig:ds1_run_time_comparison} shows the KL divergence to the ground truth on DS1 as a function of approximate runtime.
For each row, the left, middle and right plots show the results when $\xi$ is set to the prior density, the machine precision and $1e\text{-}10$ respectively.
The results for the serial implementations are shown in the first row.
We see that the speed advantage of VBPI is reduced due to the caching trick of MCMC moves.
As $\xi$ increases, the reported KL divergences from standard MCMC runs get better and VBPI is slower than MCMC when $\xi=1e\text{-}10$.
When parallelization is enabled, as shown in the second row, the advantage of VBPI is recovered.
Compared to MCMC, the performance of VBPI is not much affected by the choice of $\xi$ which is due to the large subsplit support that provides a better coverage of the ground truth samples (however, this also makes the optimization harder).
As the exact posterior here is relatively concentrated (around 3,000 unique trees in the ground truth sample), it turns out that we can reduce the size of subsplit support for faster convergence of VBPI, with only a slight decrease in the approximation accuracy.
The last row shows the results when the subsplit support were obtained using 10 replicates of short MCMC runs with 100,000 iterations (sampled every 100 iterations with a $25\%$ burn-in rate).
In this case, we can use a faster annealing schedule where $\beta_t = \min(1, 0.001+ t/50,000)$ and the convergence time is halved compared to the previous bootstrap method, albeit the KL divergence would be a bit larger when $\xi$ is small.

\begin{figure}[t!]
\begin{center}
\includegraphics[width=\textwidth]{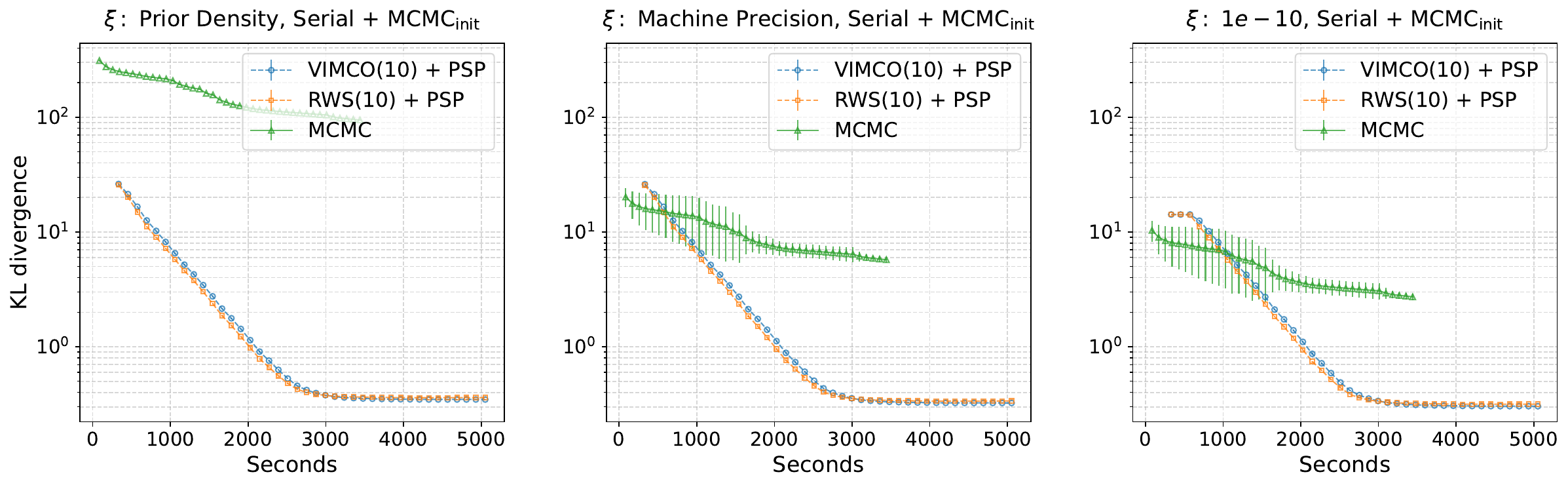}\\
\includegraphics[width=\textwidth]{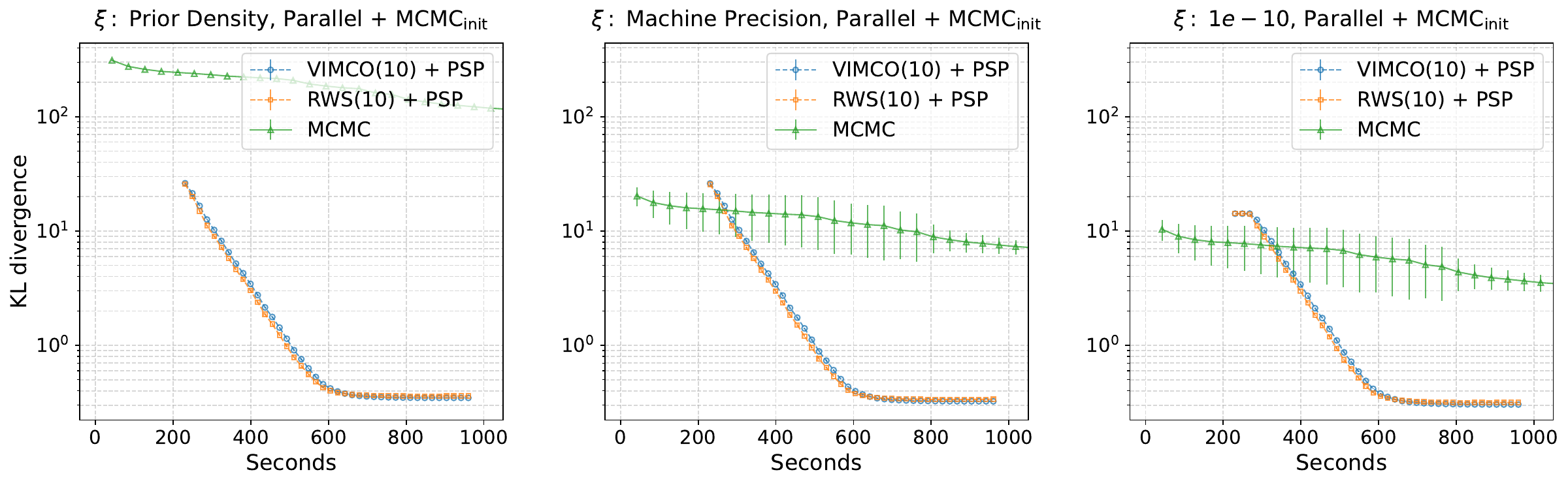}
\caption{A runtime comparison between VBPI and MCMC on a curated influenza data set with 100 sequences. ({\bf Top}): Serial implementation with short run MCMC subsplit support. ({\bf Bottom}): Parallel implementation with short run MCMC subsplit support. Each column corresponds to a different value of $\xi$. Error bars show one standard deviation over 10 independent runs.}\label{fig:flu100_run_time_comparison}
\end{center}
%\vspace{-0.2in}
\end{figure}

The advantage of VBPI should be more evident when the posterior is more diffuse.
Figure \ref{fig:flu100_run_time_comparison} shows the runtime comparison results on a curated influenza data set with 100 sequences.
Although subsplit support estimation would require some extra computation, we see that overall VBPI converges much faster than MCMC even using serial implementations.
The advantage of VBPI can be further improved when parallelization is used.
As the tree space is large and the posterior is quite diffuse (around 100,000 unique trees in the ground truth sample), when the prior density is used for $\xi$, the KL divergences of standard MCMC would be high since lots of trees in the ground truth sample would not be sampled yet (this is also a limitation of MCMC which usually requires a large number of samples for accurate posterior estimation, especially when the posterior is diffuse).
When $\xi$ increases, the KL divergences of standard MCMC runs get better but still fall behind VBPI by a large margin.

%\textcolor{blue}{Overall, compared to MCMC, VBPI would be more computationally efficient when the posterior is more diffuse, which is usually the case for data sets with many sequences.}
%We also admit that at current stage, VBPI is still underdeveloped and has many rooms for future improvement, such as more efficient optimization algorithms, better subsplilt support estimation, fast implementation and etc (see Section \ref{sec:conclusion} for more detailed discussion).

\section*{Appendix E. Heterochronous Coalescent Models}
In practice, we may have different sampling times when the observed organisms are evolving rapidly.
In this case, the standard coalescent model can be generalized to account for such \emph{heterochronous} sampling \citep{Rodrigo1999,Felsenstein1999-cm}.
Consider the ordered node times $t_{2N-1}=0<t_{2N-2}<\cdots< t_1$, which include both the sampling times for the tip nodes and the coalescent times for the internal nodes.
Let $\ell_k$ denote the number of lineages co-existing in the time interval $(t_{k}, t_{k-1})$ between node $k$ and node $k-1$.
Let $\mathcal{C}$ denote the set of indices for the internal nodes.
The density for heterochronous coalescent models can be obtained by modifying density \eqref{eq:isochronous_coalescent} as
\begin{equation}\label{eq:heterochronous_coalescent}
p(\bm{t}|N_e(t)) = \prod_{k\in \mathcal{C}}\frac{A_{k+1}}{N_e(t_{k})}\cdot\prod_{k=2}^{2N-1}\exp\left(-\int_{t_k}^{t_{k-1}}\frac{A_k}{N_e(t)}dt\right)
\end{equation}
where $\bm{t}=(t_1,\ldots,t_{2N-1}), A_k={{\ell_k}\choose{2}}$.
Note that density \eqref{eq:isochronous_coalescent} can be viewed as a special case of density \eqref{eq:heterochronous_coalescent} when $t_{k}=0,\forall k\geq N$.

\newpage

\section*{Appendix F. Consensus Tree Comparison}%\label{sec:Consensus}
\begin{figure}[h!]
\begin{center}
\includegraphics[width=0.87\textwidth]{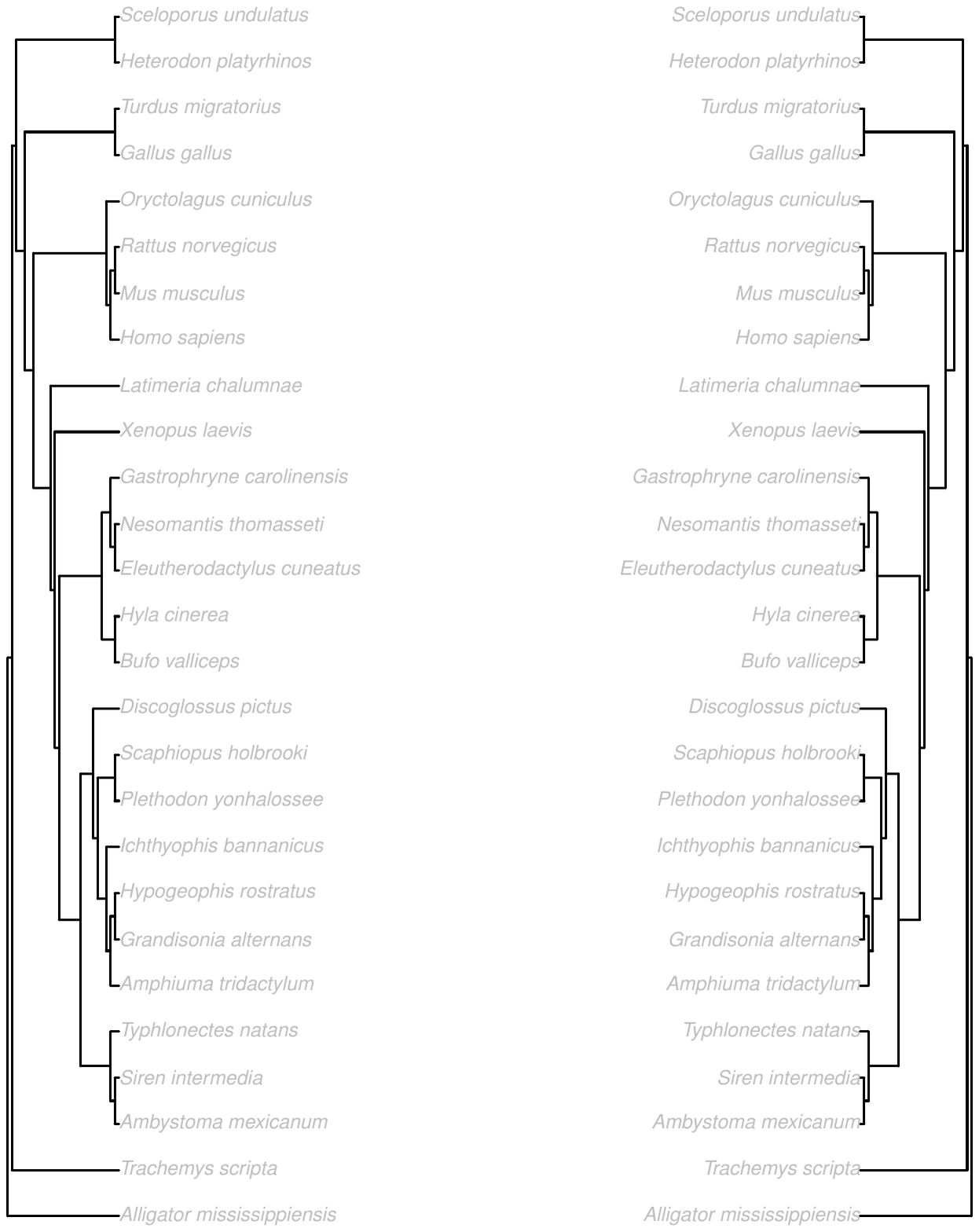}
%\hspace{2.3cm} (a)  \hspace{8cm} (b) \hspace{2.3cm}
\caption{A comparison of majority-rule consensus trees obtained from VBPI and ground truth MCMC run on DS1. ({\bf Left}): Ground truth MCMC. ({\bf Right}): VBPI ($10,000$ sampled trees). The plot is created using the \emph{treespace} \citep{treespace} R package.}\label{fig:consensus}
\end{center}
\end{figure}

\section*{Appendix G. Importance Sampling via Variational Approximations}
%\label{sec:ISpi}
In this section, we provide a detailed importance sampling procedure for marginal likelihood estimation for phylogenetic inference based on the variational approximations provided by VBPI.

\vspace{1em}
\noindent
{\bf Estimating Marginal Likelihood of Trees}
For each tree $\tau$ that is covered by the subsplit support,
\[
Q_{\bm{\psi}}(\bm{q}|\tau) = \prod_{e\in E(\tau)} p^{\mathrm{Lognormal}}\left(q_e\mid \mu(e,\tau), \sigma(e,\tau)\right)
\]
is our variational approximation to the posterior of branch lengths on $\tau$, where the mean and variance parameters $\mu(e,\tau), \sigma(e,\tau)$ are gathered from the structured variational parameters $\bm{\psi}$ as introduced in Section \ref{sec:variational_parameterization}. Therefore, we can estimate the marginal likelihood of $\tau$ using importance sampling with $Q_{\bm{\psi}}(\bm{q}|\tau)$ being the importance distribution as follows
\[
p(\bm{Y}|\tau) = \mathbb{E}_{Q_{\bm{\psi}}(\bm{q}|\tau)} \frac{p(\bm{Y}|\tau, \bm{q}) p(\bm{q})}{Q_{\bm{\psi}}(\bm{q}|\tau)} \simeq \frac1M \sum_{j=1}^M\frac{p(\bm{Y}|\tau, \bm{q}^j) p(\bm{q}^j)}{Q_{\bm{\psi}}(\bm{q}^j|\tau)} \; \text{ with } {\bm{q}}^j \mathrel{\overset{\makebox[0pt]{\mbox{\normalfont\tiny\sffamily iid}}}{\sim}} Q_{\bm{\psi}}(\bm{q}|\tau).
\]

\vspace{1em}
\noindent
{\bf Estimating Model Evidence}
Similarly, we can estimate the marginal likelihood of the data as follows
\[
p(\bm{Y}) = \mathbb{E}_{Q_{\bm{\phi},\bm{\psi}}}(\tau, \bm{q}) \frac{p(\bm{Y}|\tau, \bm{q})p(\tau, \bm{q})}{Q_{\bm{\phi}}(\tau)Q_{\bm{\psi}}(\bm{q}|\tau)} \simeq \frac1K\sum_{j=1}^K\frac{p(\bm{Y}|\tau^j, \bm{q}^j) p(\tau^j, \bm{q}^j)}{Q_{\bm{\phi}}(\tau^j)Q_{\bm{\psi}}(\bm{q}^j|\tau^j)}  \;\text{ with }\; \tau^j, {\bm{q}}^j \mathrel{\overset{\makebox[0pt]{\mbox{\normalfont\tiny\sffamily iid}}}{\sim}} Q_{{\bm{\phi}},{\bm{\psi}}}(\tau,\;{\bm{q}}).
\]
In our experiments, we use $K=1,000$. When taking a log transformation, the above Monte Carlo estimate is no longer unbiased (for the evidence $\log p(\bm{Y})$). Instead, it can be viewed as one sample Monte Carlo estimate of the lower bound
\begin{equation}
L^K(\bm{\phi},{\bm{\psi}}) = \mathbb{E}_{Q_{\bm{\phi},{\bm{\psi}}}(\tau^{1:K},\; \bm{q}^{1:K})}\log\left(\frac1K\sum_{i=1}^K\frac{p(\bm{Y}|\tau^i, \bm{q}^i) p(\tau^i, \bm{q}^i)}{Q_{\bm{\phi}}(\tau^i)Q_{\bm{\psi}}(\bm{q}^i|\tau^i)}\right)\leq \log p(\bm{Y})
\end{equation}
whose tightness improves as the number of samples $K$ increases. Therefore, with a sufficiently large $K$, we can use the lower bound estimate as a proxy for Bayesian model selection.

%\noindent
%{\bf Theorem} {\it Let $u,v,w$ be discrete variables such that $v, w$ do
%not co-occur with $u$ (i.e., $u\neq0\;\Rightarrow \;v=w=0$ in a given
%dataset $\dataset$). Let $N_{v0},N_{w0}$ be the number of data points for
%which $v=0, w=0$ respectively, and let $I_{uv},I_{uw}$ be the
%respective empirical mutual information values based on the sample
%$\dataset$. Then
%\[
%	N_{v0} \;>\; N_{w0}\;\;\Rightarrow\;\;I_{uv} \;\leq\;I_{uw}
%\]
%with equality only if $u$ is identically 0.} \hfill\BlackBox
%
%\noindent
%{\bf Proof}. We use the notation:
%\[
%P_v(i) \;=\;\frac{N_v^i}{N},\;\;\;i \neq 0;\;\;\;
%P_{v0}\;\equiv\;P_v(0)\; = \;1 - \sum_{i\neq 0}P_v(i).
%\]
%These values represent the (empirical) probabilities of $v$
%taking value $i\neq 0$ and 0 respectively.  Entropies will be denoted
%by $H$. We aim to show that $\fracpartial{I_{uv}}{P_{v0}} < 0$....\\
%
%{\noindent \em Remainder omitted in this sample. See http://www.jmlr.org/papers/ for full paper.}

\vskip0.2in
\bibliography{deep-vbpi}

\end{document}